\documentclass{article}

\usepackage{microtype}
\usepackage{graphicx}
\usepackage{subcaption}
\usepackage{booktabs} 
\usepackage{overpic}

\usepackage{hyperref}

\usepackage[accepted]{icml2026_fogen}

\usepackage{amsmath}
\usepackage{amssymb}
\usepackage{mathtools}
\usepackage{amsthm}

\usepackage[capitalize,noabbrev]{cleveref}

\theoremstyle{plain}
\newtheorem{theorem}{Theorem}[section]

\newtheorem{lemma}[theorem]{Lemma}

\theoremstyle{definition}
\newtheorem{definition}[theorem]{Definition}

\theoremstyle{remark}

\usepackage[textsize=tiny]{todonotes}

\icmltitlerunning{Understanding the Copying Behavior of DMD Students}

\begin{document}

\twocolumn[
  \icmltitle{Why Are DMD Students Lazy?\\ Understanding the Copying Behavior in Few-Step Distillation}



  \icmlsetsymbol{eqsup}{*}

\begin{icmlauthorlist} 
    \icmlauthor{Shucheng Li}{ox}
    \icmlauthor{Iolo Jones}{ox}
    \icmlauthor{Alexander Tong}{aithyra,eqsup}
    \icmlauthor{Michael M. Bronstein}{ox,aithyra,eqsup}
\end{icmlauthorlist}

\icmlaffiliation{ox}{Department of Computer Science, University of Oxford, Oxford, UK}
\icmlaffiliation{aithyra}{AITHYRA, Research Institute for Biomedical AI, Vienna, AT}

\icmlcorrespondingauthor{Shucheng Li}{shucheng.li@cs.ox.ac.uk}

  \icmlkeywords{Diffusion Model, Distillation, ICML}

  \vskip 0.3in
]



\printAffiliationsAndNotice{* Equal Supervision.}  

\begin{abstract}
  Distribution Matching Distillation (DMD) compresses pretrained diffusion models into efficient few-step generators by aligning their noised distributions across all scales. In principle, such distribution-level supervision remains agnostic to the teacher’s specific noise--data pairings; this provides the student the freedom to remap latent noise, a behavior consistently observed in low-dimensional settings. Surprisingly, we find that in high-dimensional settings, distilled students spontaneously reproduce the teacher’s original noise--data pairings—a phenomenon we term copying. We demonstrate that copying is neither a byproduct of adversarial objectives nor a result of teacher memorization. Instead, our evidence suggests that copying is an emergent property arising from the limited geometric freedom of the student model during high-dimensional distillation.
\end{abstract}

\section{Introduction}
\label{Sec1. Introduction}
Diffusion models have fueled significant advances across images, videos, and text~\cite{DDIM, DDPM, DiffInstruct, ImagenVideo, VSDvideogen, LDM}. However, high-quality sampling through stochastic differential equations (SDEs) is computationally expensive. Distillation methods, such as Distribution Matching Distillation (DMD)~\cite{DMD1, DMD2}, bypass this bottleneck by training single-step students to match the teacher’s distribution. 

While recent studies have uncovered reproducibility~\cite{emergenceandreproducibility}, where multiple \textit{trained }diffusion models with different architectures produce the same noise--data pairings, the noise--data pairings of \textit{distilled} students remain largely unexplored. In this work, we identify an unexpected behavior which we term \textbf{copying} for distillation in high-dimensional settings: students faithfully reproduce the teacher’s noise--data pairings, despite being trained on an objective that is \textit{pairing-indifferent}.

We demonstrate that copying is neither a trivial artifact of teacher memorization~\cite{whyDMdontmemorize} nor a consequence of auxiliary losses. Instead, we posit that copying emerges from the constrained degrees of freedom inherent in high-dimensional manifolds. Our experiments reveal that copying is most pronounced at the distribution boundaries and correlates with the convergence level of the teacher. These patterns emphasize that student distillation dynamics are fundamentally governed by the geometric structure of the teacher's mapping, providing a new perspective on the interaction between high-dimensional data geometry and generative distillation.

Our main contributions are as follows:
\begin{itemize}
\item \textbf{Observing and Quantifying Copying:} We establish copying as a non-trivial dynamical behavior by proving that the DMD objective is in principle invariant to the noise--data pairings learned by the student. We propose \textit{Pairing Inefficiency} ($\Delta_E$), a scale-invariant metric, allowing fair quantitative comparison of strengths of copying across datasets of disparate scales and resolutions.

\item \textbf{Analyzing triggering factors of copying:} We demonstrate a sharp empirical divergence in copying behavior between low- and high-dimensional settings. Through ablations, we rule out auxiliary objectives and teacher memorization as primary drivers for copying. We provide evidence on macroscopic and microscopic scales that characterize copying as a consequence of emergent high-dimensional geometric constraints. 
\end{itemize}

\section{Background}
\label{Sec2. Background}
\subsection{Diffusion Models}
\label{Sec2.1 Diffusion Models}
Diffusion models generate samples from an unknown data distribution $p_{\text{data}}$ by reversing a learned stochastic process that transforms $p_{\text{data}}$ into a prior noise distribution $p_{\varepsilon}$, typically approximately a scaled isotropic Gaussian $\mathcal{N}(0, \sigma^2 \mathbf{I})$. This process is characterized by a pair of forward and backward stochastic differential equations~\cite{song2021scorebased}.

In the Variance Exploding (VE) framework, the forward SDE evolves $x_0 \sim p_{\text{data}}$ into increasingly noisy latent variables $x_t$ over the time interval $t \in [0, T]$ according to a noise schedule $\sigma(t)$:

$$
\mathrm{d} x_t = \sqrt{2 \dot{\sigma}(t) \sigma(t)} \mathrm{d} \mathbf{W}_t,
$$

where $\mathbf{W}_t$ denotes a standard multi-dimensional Wiener process. The transition kernel for this process is given by $p_t(x_t | x_0) = \mathcal{N}(x_t; x_0, \sigma^2(t) \mathbf{I})$. The marginals $\{p_t\}_{t \in [0, T]}$ define a \textbf{probability path} starting at $p_0 = p_{\text{data}}$ and terminating at a distribution $p_T \stackrel{d} {\approx} p_\varepsilon :=\mathcal{N}(0,  \sigma^2(T) \mathbf{I})$, provided the noise variance scale $\sigma^2(T)\gg \sigma^2 _{\text{data}}$ the variance scale of the data distribution.

To sample from $p_{\text{data}}$, the probability path can be inverted by initializing $x_T=z \sim \mathcal N(0,\sigma^2(T)\mathbf{I})$ and solving a corresponding backward SDE, or its equivalent deterministic Probability Flow ODE:

\begin{equation}
\mathrm{d} x_t = -\dot{\sigma}(t) \sigma(t) s(x_t,t) \mathrm{d} t,
\label{PFODE}
\end{equation}

where $s(x_t,t) := \nabla_x \log p_t(x_t)$ represents the Stein score function for $\{p_t\}$. Since the ground-truth score function is inaccessible, it is typically estimated by a learnable time-conditioned neural network $s_\theta(x_t, t)$. While other frameworks such as Variance Preserving (VP) or Flow Matching~\cite{flowmatching} can be similarly formulated under this SDE template with differing drift and diffusion coefficients~\cite{EDM, principlesofDM}, we restrict attention to the VE framework with $\sigma(t)=t$ for ease of presentation.

Under the VE framework, since the transition kernel $p_t\left(x_t | x_0\right)$ is Gaussian, the Stein score $s\left(x_t, t\right)$ relates to the posterior expectation of the clean signal via Tweedie's formula: $s\left(x_t, t\right)=(\mathbb{E}\left[x_0 | x_t\right]-x_t)/{t^2}.$ This identity allows a reparametrization of score to stabilize learning. In particular, we typically train a neural network $x_{0, \theta}\left(x_t, t\right)$ to approximate $\mathbb{E}\left[x_0 | x_t\right]$, then set $$s_\theta\left(x_t, t\right)=\frac{x_{0, \theta}\left(x_t, t\right)-x_t}{t^2}.$$

In image dataset learning tasks, this denoiser $x_{0,\theta}$ is often implemented as a UNet with additional pre-conditioning to ensure learning stability~\cite{DMD2, EDM}.

\subsection{Distribution Matching Distillation}
\label{Sec2.2 Distribution Matching Distillation}
Distribution Matching Distillation~\cite{DMD1, DMD2} aims to overcome the computational bottleneck of multi-step inference by distilling the teacher diffusion model $s(x_t,t)$ into a single-step student generator $G_\theta(z)$, by aligning the student probability path $p_{\theta,t}$, defined as $$p_{\theta, t} := \textbf{Law}(G_\theta(z) + \sigma(t)\varepsilon)$$where $\varepsilon\sim \mathcal N(0,\mathbf I)$ is independent of $z\sim \mathcal N(0,\sigma^2(T)\mathbf I)$, with the teacher's probability path $p_t$ across all noise levels~$\sigma(t)$. This is achieved by minimizing a time-weighted integral of the Kullback-Leibler (KL) divergence, referred to as the \textbf{Distribution Matching (DM) loss}:

\begin{equation}
\label{DM loss}
L_{\text{DM}}(\theta) := \int w(t) \text{KL}(p_{\theta, t} \| p_t) \, \mathrm dt. 
\end{equation}

The gradient of this objective with respect to the student parameters $\theta$ can be expressed in closed form:

\begin{equation}
\label{DM gradient}
\nabla_\theta L_{\text{DM}}(\theta) = \mathbb{E}_{t, z, \varepsilon} \left[ w(t) \left( s_\psi(x_t, t) - s(x_t, t) \right) \frac{\partial G_\theta(z)}{\partial \theta} \right],
\end{equation}

where $x_t = G_\theta(z) + \sigma(t)\varepsilon$. Here, $s(x_t, t)$ is the pre-trained teacher score, and $s_\psi(x_t, t)$ is a learnable student score neural network that approximates the intractable student score~$\nabla_x \log ((G_\theta(z)\ast \mathcal N(0,t^2\mathbf I))(x))$ during distillation.

We initialize the student generator and student score with teacher’s weights, specifically setting $G_\theta(z) := x_0(z, T)$, and $s_\psi(x_t ,t)=(x_0(x_t,t)-x_t)/t^2$, providing a strong starting point for the distillation procedure. To maintain stability during distillation, the student score $s_\psi$ and generator $G_\theta$ are updated via a joint training schedule. In each iteration, $s_\psi$ is updated by minimizing a standard denoising score matching loss
$$
\min_\psi \mathbb{E}_{t, z, \varepsilon} \left[ \| s_\psi(x_t, t) +\frac{\varepsilon}{\sigma(t)} \|^2 \right],
$$
where $x_t=G_\theta(z)+\sigma(t)\varepsilon$, to ensure that $s_\psi$ accurately tracks the score of the current student probability path $p_{\theta,t}$. The student generator $G_\theta$ is updated every 5 iterations using the gradient from Equation~\ref{DM gradient}. Auxiliary regression or adversarial objectives may be added~\cite{DMD2, DMD1}.

\begin{figure*}[ht]
  \vskip 0.2in
  \begin{center}
    \centerline{\includegraphics[width=1.8\columnwidth]{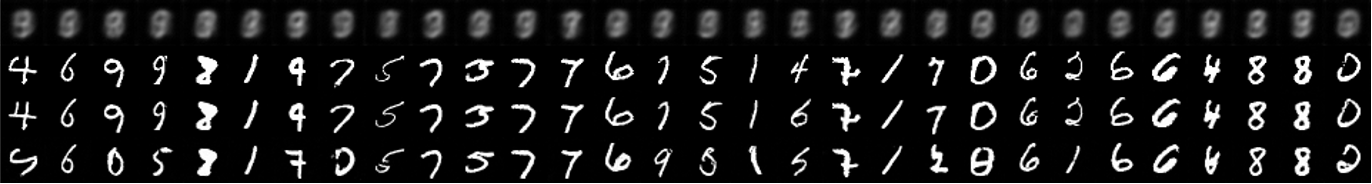}}
    
    \vskip 0.1in 
    \centerline{\includegraphics[width=1.8\columnwidth]{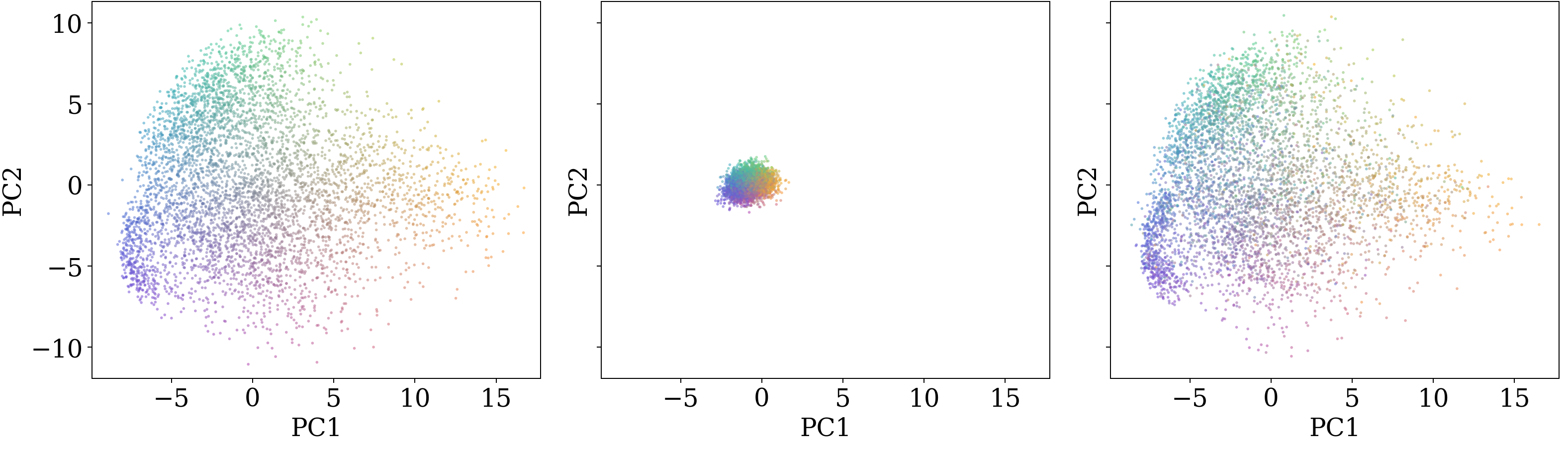}}
    \caption{
    \textbf{Significant copying in high-dimensional settings.} The distilled student on unconditional MNIST exhibits strong copying behavior. \textbf{Top:} Thirty image quadruples generated from random initial noise seeds \(z\). From top to bottom: teacher 1-step samples \(\Phi_1(z)\), teacher 8-step samples \(\Phi_8(z)\), teacher 32-step samples \(\Phi_{32}(z)\), and student 1-step samples \(G(z)\). \textbf{Bottom:} Visualization of 2000 triplets \((\Phi_8(z),\Phi_1(z),G(z))\) projected onto the two leading principal components of \(\Phi_8(z)\). Points generated from the same noise seed are assigned the same color across panels, showing that $G(z)$ occupies the same manifold location as the multi-step teacher \(\Phi_8(z)\). Pairing inefficiency is low at \(\Delta _E\approx 0.0367.\)
    }
    \label{uncondMNISTpairresult}
  \end{center}
  \vskip -0.2in
\end{figure*}

\section{The Copying Behaviour}
\label{Sec3. Copying}

We observe an unexpected phenomenon of Distribution Matching Distillation (DMD) in high-dimensional settings: a student model tends to faithfully reproduce the teacher's noise-data pairings pointwise, even though it is trained exclusively to match the teacher's \textit{distribution}. As illustrated in Figure~\ref{uncondMNISTpairresult}, over the unconditional MNIST learning task, the student generator $G_\theta(z)$ aligns very closely with the multi-step teacher target $\Phi_K(z)$ obtained by taking $K$ Euler steps\footnote{We fix $K=8$ across all experiments, as this setting suffices for high-quality synthesis, with larger values offering only marginal improvements in sample quality.}backwards with second-order corrections. We term this spontaneous alignment behavior \textbf{copying}.

\subsection{Measuring Copying by Pairing Inefficiency}
\label{Sec3.1. Measuring Copying by Pairing Inefficiency}

To formally quantify the strength of copying and compare it across different scales and dimensions, we introduce a scale-invariant measure $\Delta_E(\Phi_K, G_\theta)$ called pairing inefficiency. 

\begin{definition}[Pairing Inefficiency]
Let the initial noise be $z \sim p_{\varepsilon} = \mathcal{N}(0, \sigma^2(T) \mathbf{I})$. Let $p_{\Phi} = (\Phi_K)_{\#} p_{\varepsilon}$ and~$p_G = G_{\#} p_{\varepsilon}$ be the distributions learned by the teacher and student. The \textbf{Optimal Transport (OT) energy} and the \textbf{Distillation Transport (DT) energy} are defined as:
\begin{align}
    E_{OT}(\Phi_{K}, G) &:= \min_{\pi \in \Gamma(p_{\Phi}, p_G)} \int \|x - y\|_2^2 \, d\pi(x, y) \\
    E_{DT}(\Phi_{K}, G) &:= \int \|\Phi_K(z) - G(z)\|_2^2 \, d p_{\varepsilon}(z)
\end{align}
where $\Gamma(p_{\Phi}, p_G)$ is the class of all couplings with marginals~$p_{\Phi}$ and $p_G$. We define the \textbf{pairing inefficiency} as
\begin{equation}
    \Delta_E(\Phi_{K},G) := \frac{E_{DT}(\Phi_K,G)}{E_{OT}(\Phi_K,G)} - 1.
\end{equation}
\end{definition}

A small inefficiency $\Delta_E \approx 0$ implies strong copying, while a larger inefficiency $\Delta_E$ implies the teacher pairings are \textit{remapped}. In practice, we use a consistent Monte Carlo estimator $\Delta_E^{(N)}$ with $N=1000$ samples. Formal justifications including the non-negativeness and scale-invariance of $\Delta_E$ are provided in Appendix \ref{C.Derivations}.

\subsection{Copying is Not Required for Successful Distillation}
\label{Sec3.2 Copying is Not Required for Successful Distillation}

Crucially, we emphasize that copying is not a theoretical necessity of the DMD objective. The Distribution Matching loss $L_{\mathrm{DM}}$ is \textit{pairing-indifferent}: it only penalizes the discrepancy between the student distribution $p_G$ and the teacher's $p_\Phi$, remaining agnostic to the noise--data pairings the student learned.

\begin{lemma}
\label{lem:indifference}
Let $G_\theta$ and $G_{\theta'}$ be two student generators. Whenever $G_\theta(z) \stackrel{d}{=} G_{\theta'}(z)$, it follows that~$L_{DM}(\theta) = L_{DM}(\theta')$, even though in general $\nabla_\theta L_{DM}(\theta) \neq \nabla_\theta L_{DM}(\theta')$. Consequently, the stochastic optimization dynamics can drive students to converge toward pairings with vastly different inefficiencies $\Delta_E$ despite achieving similar distributional fidelity. 
\end{lemma}

\begin{figure*}[ht]
  \vskip 0.2in
  \begin{center}
    \centerline{\includegraphics[width=2\columnwidth]{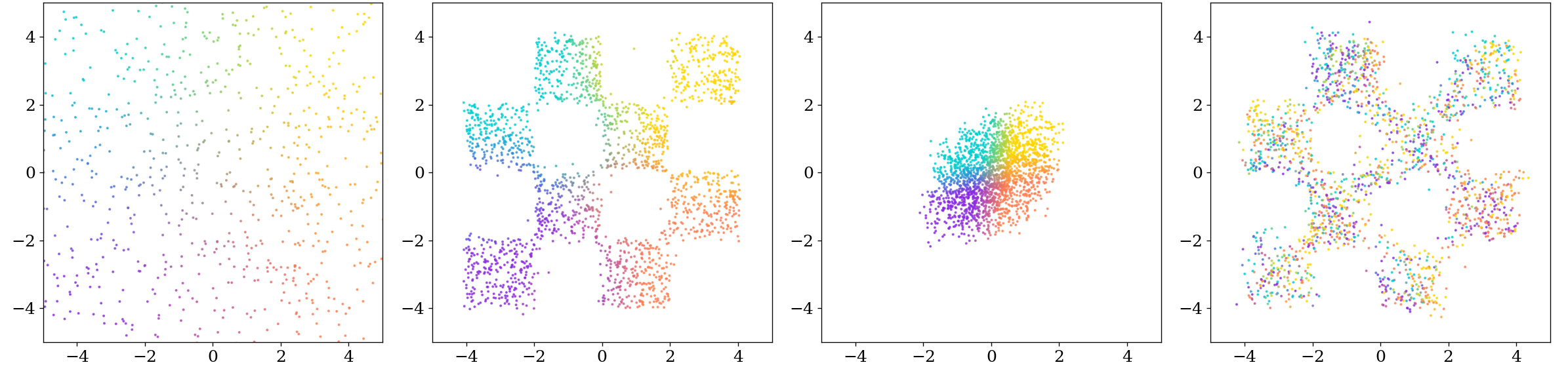}}
   \caption{
\textbf{Copying does not necessarily occur, and rarely occurs in low-dimensional space.} The distilled student on synthetic chessboard dataset exhibits strong remapping behavior. The dataset is a 2D chessboard dataset embedded in the first two coordinates of a four-dimensional ambient space. Panels from left to right show the initial Gaussian noise \(z\), the teacher 8-step samples \(\Phi_8(z)\), the teacher 1-step samples \(\Phi_1(z)\), and the student one-step samples \(G(z)\). Points generated from the same initial noise seed \(z\) are assigned the same color across panels. Visualization is obtained by projection onto the first two coordinates. The pairing inefficiency is high at \(\Delta _E\approx 8.55\).
}
    \label{2dchessboardpairing}
  \end{center}
\end{figure*}

\begin{figure}[t]
    \centering

    \begin{subfigure}{0.48\columnwidth}
        \centering
        \includegraphics[width=\linewidth]{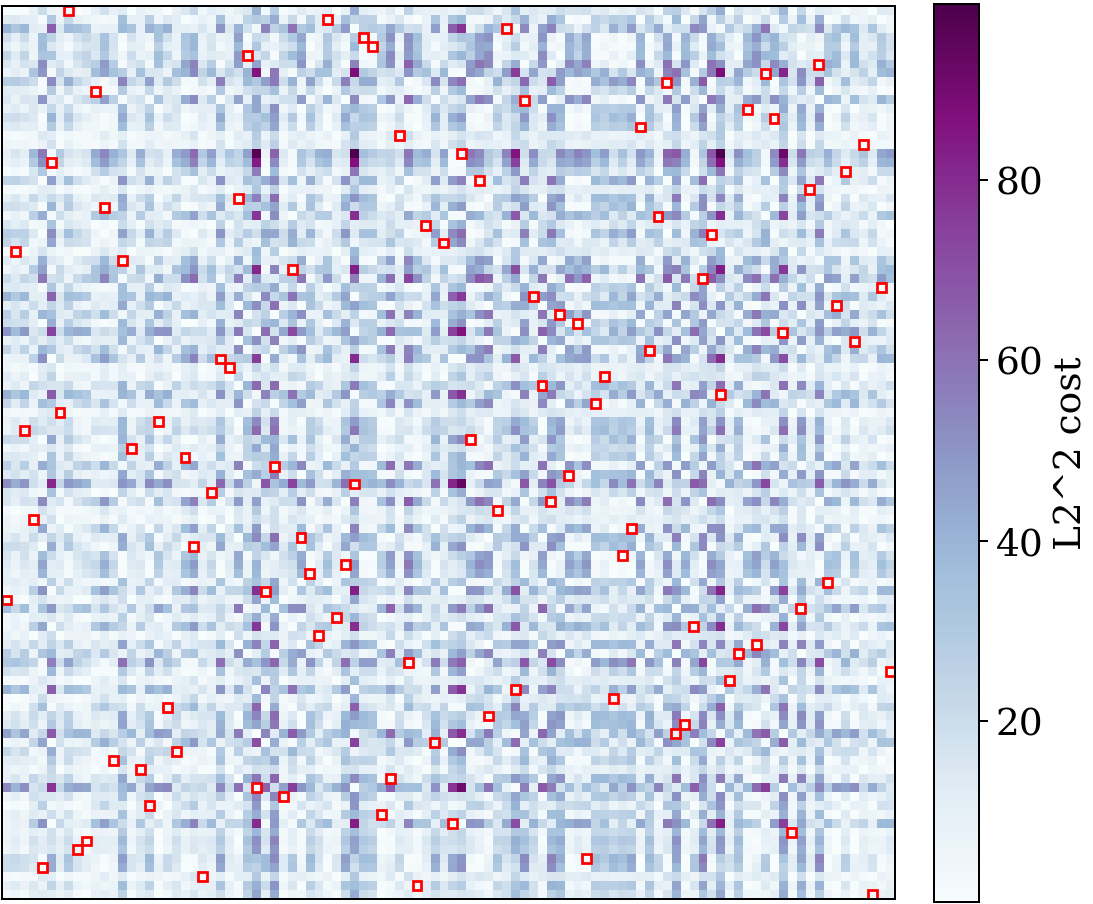}
        \caption{2D Chessboard}
        \label{fig:copy_chessboard}
    \end{subfigure}
    \hfill
    \begin{subfigure}{0.48\columnwidth}
        \centering
        \includegraphics[width=\linewidth]{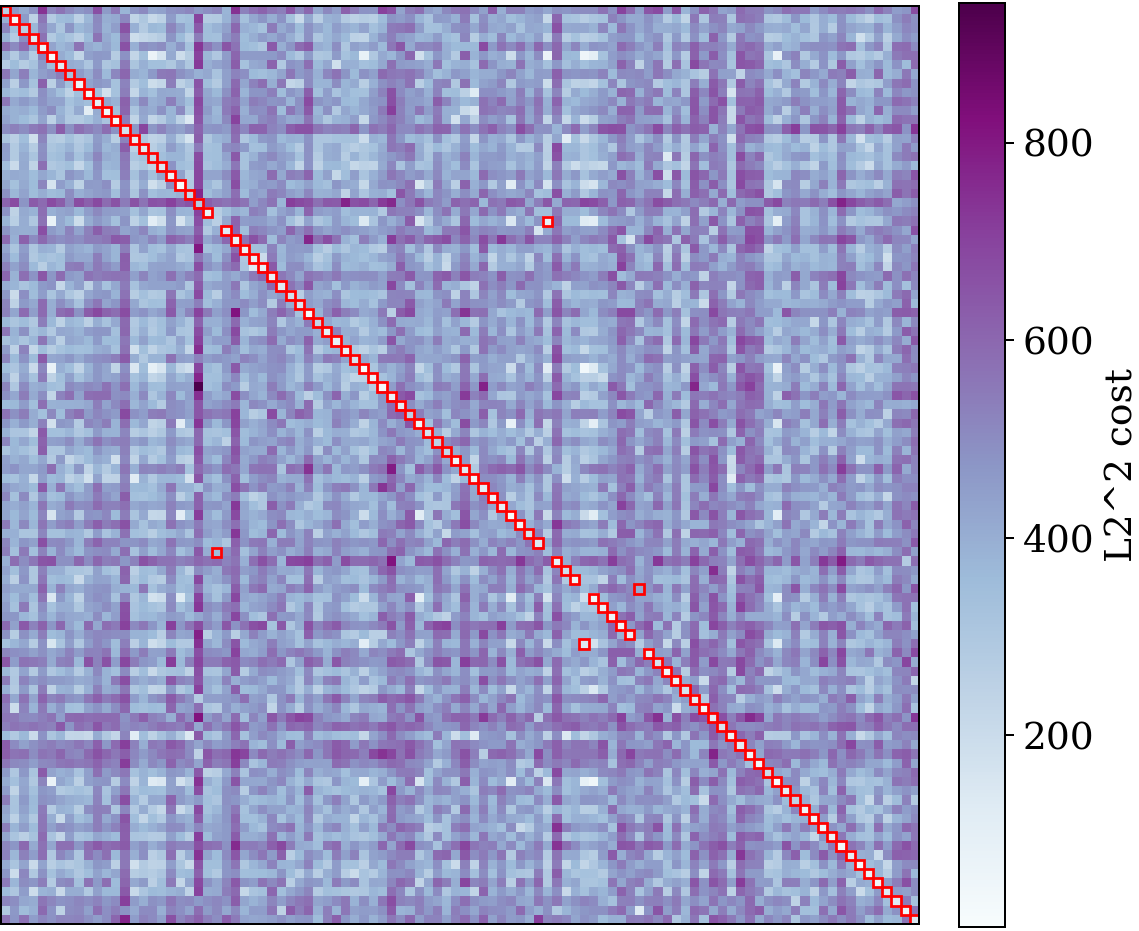}
        \caption{Unconditional MNIST}
        \label{fig:copy_mnist}
    \end{subfigure}

    \caption{
    \textbf{Copying is more pronounced in high-dimensional settings. } The heatmaps represent the pairwise squared $L_2$ distances between teacher-generated images $\{\Phi_K(z_i)\}_{i=1}^{100}$ and student-generated images $\{G(z_j)\}_{j=1}^{100}$ for the 2D chessboard dataset (left) and the unconditional MNIST dataset (right). The horizontal and vertical axes denote the teacher and student image indices, $i$ and $j$, respectively. The Optimal Transport (OT) pairing is highlighted with red boxes, while the Distillation Transport (DT) pairing corresponds to the diagonal. In the chessboard experiment, students mainly remaps the teacher's pairings, resulting in high pairing inefficiency ($\Delta_E \approx 8.55$). Conversely, in the unconditional MNIST experiment, the student mainly copies the pairings, exhibiting very low pairing inefficiency ($\Delta_E \approx 0.0367$).
}
    
    \label{2dchess and UncondMNIST pair comparison}
\end{figure}

A particular consequence of Lemma~\ref{lem:indifference} is that minimizing the distillation objective in Equation \ref{DM loss} does not require the student to satisfy $G_\theta(z) \approx \Phi_K(z)$ for most $z$ hence achieving low pairing inefficiency $\Delta _E\approx 0$. A student could match the teacher's distribution perfectly while arbitrarily remapping or reflecting the noise-to-data manifold. The proof of this lemma, along with an intuitive example where a non-copying student attains a global minimum of Objective \ref{DM loss}, is provided in Appendix~\ref{C.Derivations}.

\subsection{Copying Rarely Occurs in Low Dimensions}
\label{Sec3.3. Copying Rarely Occurs in Low Dimensions}

We have theoretically established that in principle a distilled student is not required to copy the teacher's noise-data pairings and to exhibit low pairing inefficiency. 

Indeed, we empirically observe that on low-dimensional datasets, copying is rarely present and pairing inefficiency is high. For example, we define $p_{\text{data}}$ as a two-dimensional $4 \times 4$ chessboard distribution embedded within the first two coordinates of $\mathbb{R}^4$. We first train a teacher diffusion model, parameterized by an MLP with five hidden layers of width 384, and subsequently distill a single-step student generator until convergence. As shown in Figure~\ref{2dchessboardpairing} and Figure~\ref{2dchess and UncondMNIST pair comparison}, while the student successfully recovers the target distribution $p_{\text{data}}$, it exhibits significant remapping behavior with high pairing inefficiency $\Delta_E\approx 8.55$, confirming that the student has found a valid distributional fit that deviates from the teacher's original trajectories.

Similarly significant remapping is consistently observed across other standard low-dimensional synthetic datasets (see Appendix \ref{A. Copying on Other Datasets}).

\subsection{Copying Frequently Occurs in High Dimensions}
\label{Sec3.4 Copying Frequently Occurs in High Dimensions}
In stark contrast to low-dimensional settings, copying behavior emerges consistently in high-dimensional tasks. We train a clean-prediction teacher diffusion model using a standard UNet architecture on the unconditional MNIST dataset consisting of $70,000$ digit images. The teacher is trained from scratch for $8192$ iterations, after which we distill a single-step student generator until convergence. 

Compared to the toy settings, the distilled student on conditional MNIST is now visually substantially more prone to copying (see Figure~\ref{uncondMNISTpairresult} and Figure~\ref{2dchess and UncondMNIST pair comparison}), achieving an extremely low pairing inefficiency $\Delta_E \approx 0.0367$.

We also observe that such strong copying in high dimensional settings is not exclusive to natural images; it is equally prevalent in artificial high-dimensional datasets lacking natural image structures, as well as in conditional generation settings. For a comprehensive overview of these cases, see Appendix~\ref{A. Copying on Other Datasets}.

\begin{figure*}[ht]
    \centering
    \begin{minipage}{0.49\textwidth}
        \centering
        \includegraphics[width=\linewidth]{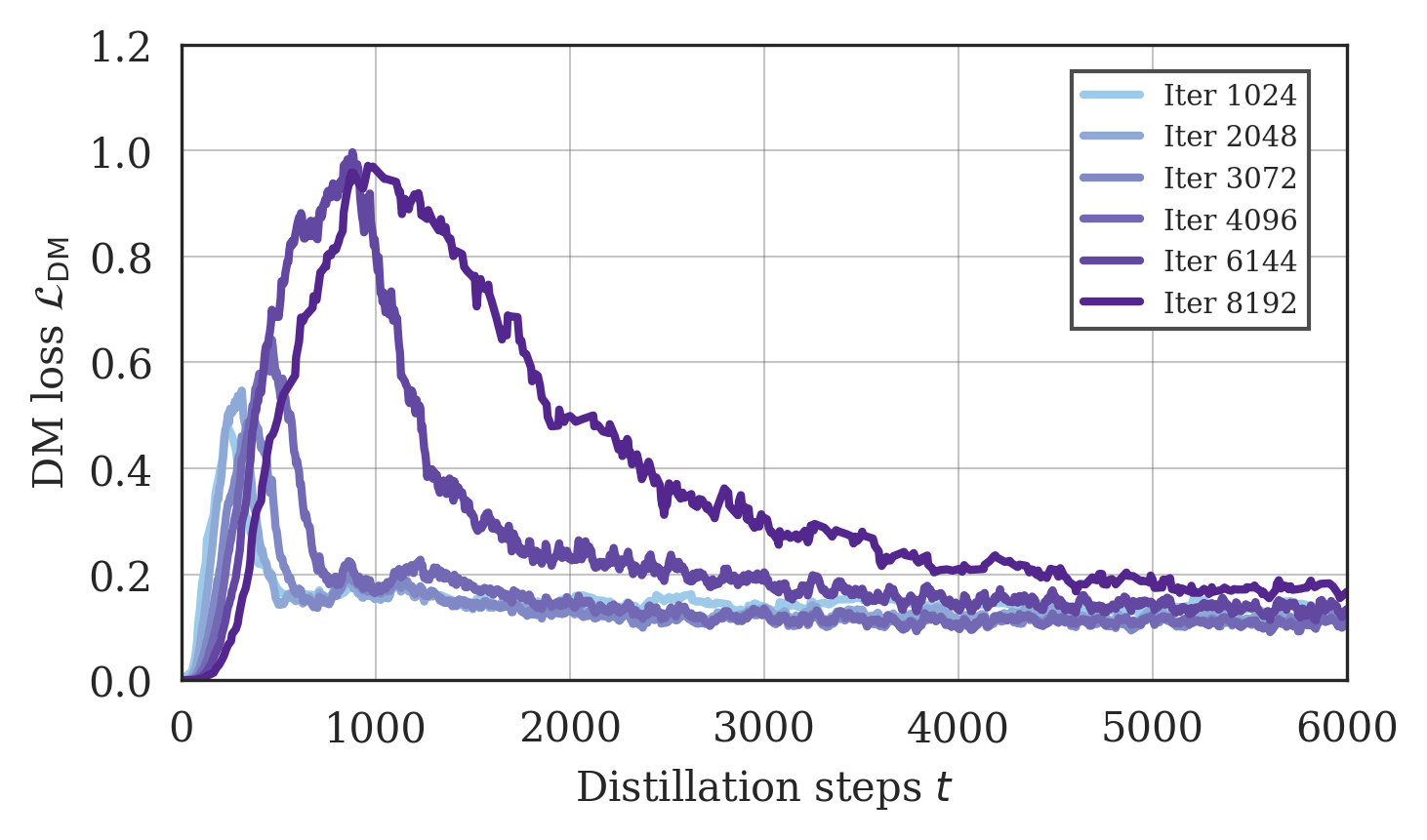}
    \end{minipage}
    \hfill 
    \begin{minipage}{0.49\textwidth}
        \centering
        \includegraphics[width=\linewidth]{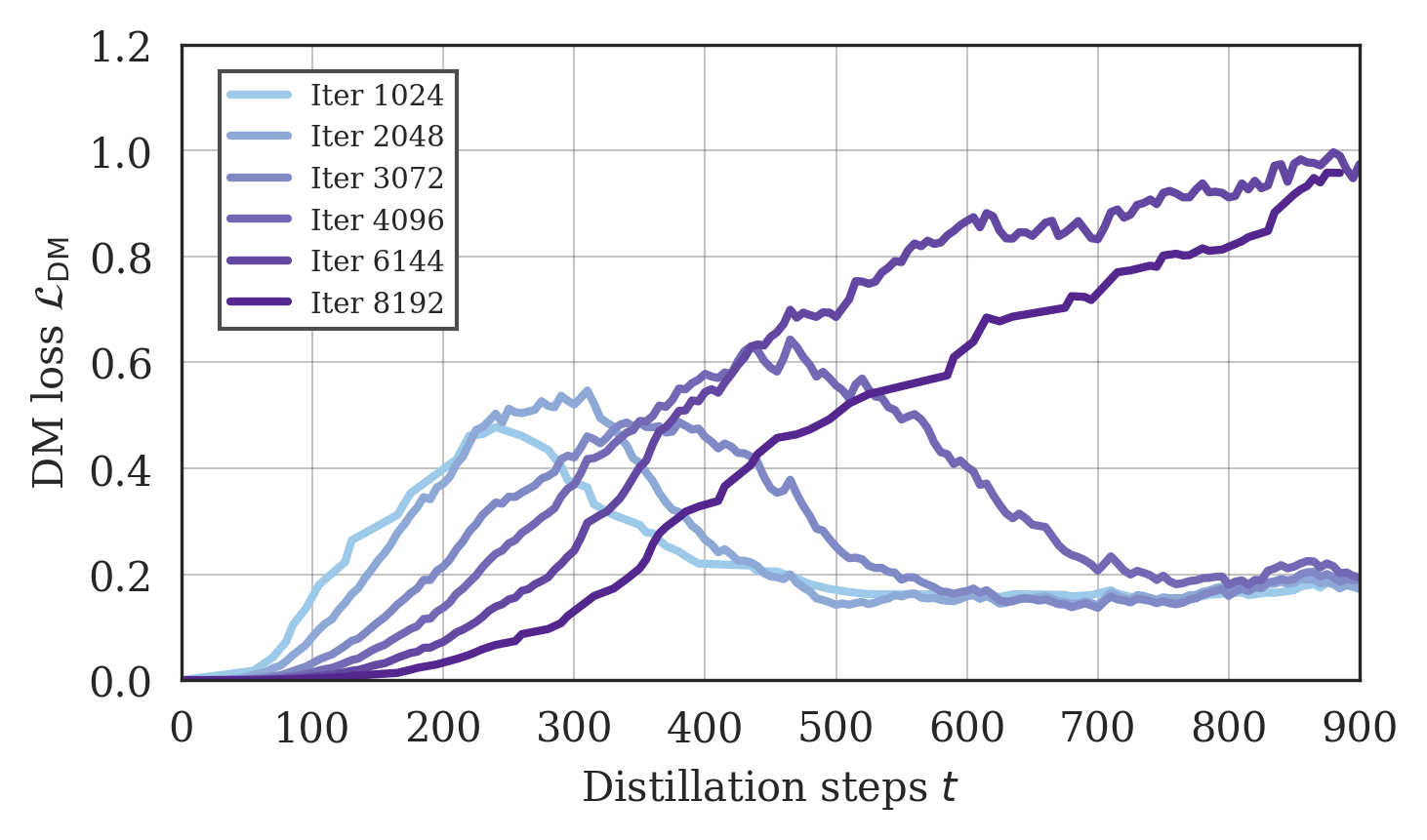}
    \end{minipage}
    
    \vspace{0.1in} 
    \caption{
   \textbf{ Two stages of Distribution Matching Distillation. } We plot the distribution matching loss of all students (initiated at different teacher snapshots) on the unconditional MNIST dataset, throughout distillation. In all cases, the DM loss exhibits a characteristic two-stage increase-decrease evolution. Students are initialized from teacher checkpoints trained for 1024, 2048, 3072, 4096, 6144, and 8192 iterations. All students are distilled for 50K iterations. The left panel shows the loss evolution up to 6K distillation iterations, and the right panel zooms into the first 900 iterations.
    }

    \label{MNIST_dm_loss}
    \vspace{-0.1in}
\end{figure*}

The consistent selection of copying solutions with $\Delta_E \approx 0$ over other valid remappings is thus an emergent property that warrants further investigation.

\section{Analyzing the Mechanisms Behind Copying}
\label{Sec 4. Analyzing the Mechanisms Behind Copying}

In this section, we analyze the mechanisms behind the copying behavior. For simplicity of discussion, we focus on the unconditional MNIST setting.

\subsection{Removal of Intuitive Explanations}
\label{Sec 4.1. Removal of Intuitive Explanations}

\subsubsection{Copying Occurs Without Adversarial Objective}

The models in \citet{DMD2} on which significant copying is observed were trained with an additional adversarial loss,
\[
L_{\mathrm{GAN}}(\gamma)
:=
\mathbb E_{\substack{x\sim p_{\mathrm{data}}\\
\varepsilon, z,
t}}
\left[
-\log D_\gamma(x_t')
+
\log D_\gamma(x_t)
\right],
\]
where
\(x_t'= G_\theta(z)+t\varepsilon,\, x_t = x+t\varepsilon.\) This adversarial loss may implicitly contribute to the copying. 

However, our experiments use neither regression loss nor the adversarial objective, the copying phenomenon nevertheless still emerges consistently. This demonstrates that neither target regression nor adversarial training is necessary for copying to occur.

\subsubsection{Copying Occurs Without Memorization}
\label{Sec 4.2. Copying Follows Geometric Complexity}
One may suspect that the teacher model memorizes the training dataset, thereby encouraging the student to simply reproduce its outputs. For a generated datapoint~\(y\), we define its \textbf{memorization distance ratio}
\[
r(y)
:=
\frac{\|y-x^1(y)\|}{\|y-x^2(y)\|},
\]
where \(x^1(y)\) and \(x^2(y)\) denote respectively the nearest and second-nearest neighbours of \(y\) in the training dataset. A datapoint \(y\) is considered memorized if
\(
r(y)<r_{\mathrm{thres}}
\)
for some prescribed threshold \(r_{\mathrm{thres}}\)~\cite{whyDMdontmemorize}. We evaluated the trained teacher models by computing memorization distance ratios for randomly generated samples. As illustrated in Figure~\ref{teacher do not memorize}, none of the generated samples exhibit signs of memorization.

\subsection{Copying Follows Geometric Complexity}
\label{Sec 4.2. Copying Follows Geometric Complexity}

In this subsection, we present additional experiments and a more detailed analysis of when copying occurs and under what circumstances it becomes more significant.

We divide the distribution-matching distillation dynamics into two stages, as shown in Figure~\ref{MNIST_dm_loss}. In the first stage, the distribution-matching (DM) loss increases. This occurs because the initialized surrogate student score model~\(s_\psi(\cdot,t)\) gradually deviates from the teacher score~$s(\cdot,t)$ and more accurately approximates and tracks the score~\(\nabla_x
\log\left(
\left(G_\theta(z)\ast\mathcal N(0,t^2 \mathbf I)\right)(x)
\right)
\) of the current student probability path. Once~\(s_\psi\) begins consistently tracking the actual student score, distillation enters the second stage. At this point, the gradient in Equation~\ref{DM gradient} begins providing meaningful optimization signal for updating \(G_\theta\), so that $s_\psi$ gradually matches \(s\), progressively deforming student's one-step distribution toward the teacher's learned target distribution, decreasing the DM loss.

We conjecture that a higher degree of copying arises when, in the second stage of distillation, the student has limited freedom to deform its distribution while still preserving the target distribution implied by the teacher score.

To illustrate this intuition, consider a target uniform distribution \(U\) with density \(\rho=1/4\) supported on the square with vertices \((\pm1,\pm1)\). Within the interior of the square, the student may continuously perturb or remap noise--data pairings while still approximately preserving the target density. In this regime, there exists substantial geometric flexibility for transport. 

In contrast, near the boundary of the support, remapping becomes significantly more constrained. To preserve the uniform distribution while continuously deforming transport pairings near the edges, the student must satisfy additional geometric and continuity constraints induced by the boundary structure. Under limited expressiveness of a single-step generator\footnote{The student model in the chessboard experiment in contrast, is highly expressive, and could easily learn a remapped distribution. See Figure~\ref{2dchessboardpairing} and Section~\ref{Sec3.3. Copying Rarely Occurs in Low Dimensions}.}, such constrained deformations may be substantially harder to realize during optimization. Consequently, the optimization dynamics may favour preserving the teacher's original noise--data pairings in these regions, leading to stronger copying behavior. 

We provide one microscopic and one macroscopic piece of evidence in support of this conjecture.

\subsubsection{Micro level: Boundary Points Are More Likely Copied}

We observe that teacher samples further away from the bulk of training dataset are more likely to be copied by the student. In particular, Figure~\ref{BoundaryPtsMorelikelyCopied} shows that the student's relative displacement towards the teacher target,
\(
\|G(z)-\Phi_K(z)\|
-
\|\Phi_1(z)-\Phi_K(z)\|
\),
is moderately negatively correlated with the average distance between the teacher target and the training dataset,
\[
D(z)
:=
\operatorname{Avg}_{x\in\mathrm{train}}
\left(
\|\Phi_K(z)-x\|
\right).
\]

\begin{figure}[t]
    \centering

    \includegraphics[width=0.8 \columnwidth]{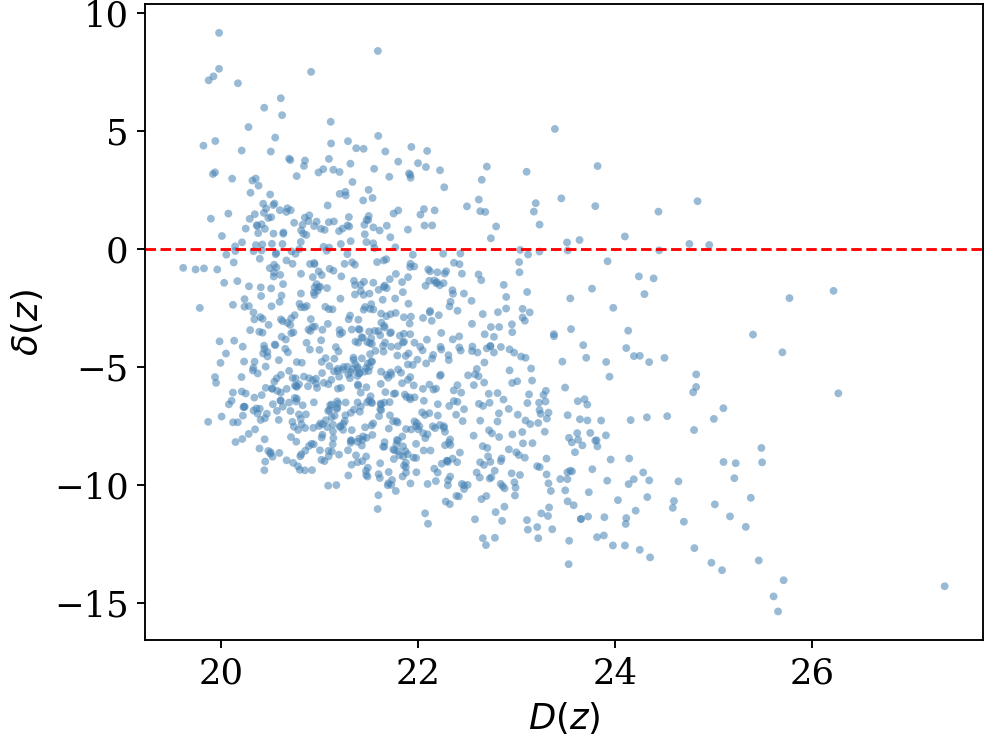}

    \caption{
    \textbf{Boundary Points are more likely copied.} For the student distilled from teacher trained on unconditional MNIST dataset for 8192 iterations, we plot with horizontal axis~$D(z)=\text{Avg}_{x\in\text{train}}(\|\Phi_K(z)-x\|)$ the average distance to training set, and vertical axis $\delta (z)=\|G(z)-\Phi_K(z)\|-\|\Phi_1(z)-\Phi_K(z)\|$ the student relative displacement towards the teacher target. Points below the red dashed line indicate aligned pairs. A larger negative displacement indicates stronger alignment.
    }
    
    \label{BoundaryPtsMorelikelyCopied}
\end{figure}

Intuitively, teacher samples $\Phi_K(z)$ with high $D(z)$ typically lie near sparse or extremal parts of the learned data manifold, where remapping while preserving the induced distribution may require satisfying stronger geometric or continuity constraints. Consequently, optimization dynamics appear to favour copying in these regions.

We note that this more pronounced copying phenomenon at extremeties of data manifold may be closely intertwined with recent insights from \citet{zhang2026generalization}, who demonstrated that trained generative models effectively capture local geometry over abundant data clusters while reverting to memorization on scarce, isolated points. We leave a formal characterization of this mechanistic correspondence and causality for future work.

\subsubsection{Macro Level: Longer Trained Teachers Are More Likely Copied}

We observe that, although the teacher models do not memorize, at a macroscopic level the degree of student copying increases with the number of iterations used to train the teacher model.

\begin{figure}[t]
    \centering
        \includegraphics[width=0.46\textwidth]{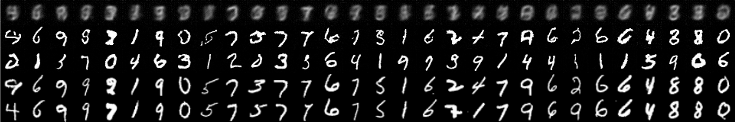}
        
        \vskip 0.08in
        
        \includegraphics[width=0.46\textwidth]{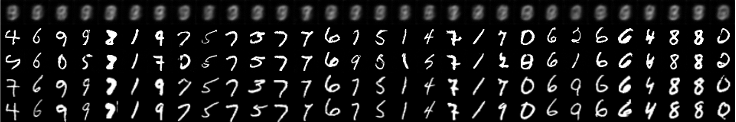}
        
        \vskip 0.08in
        
        \begin{subfigure}{0.23\textwidth}
        \centering
        \includegraphics[width=\linewidth]{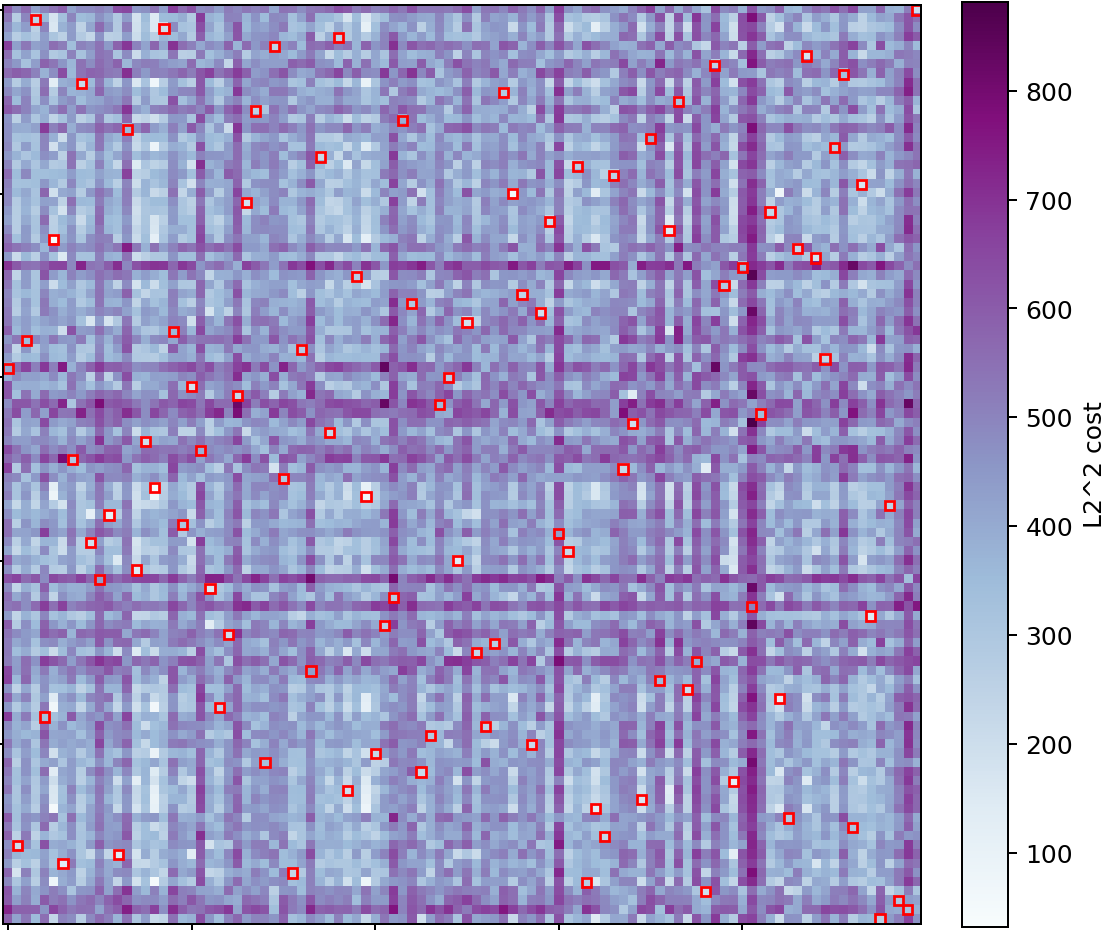}
        \caption{Teacher 1k Iterations}
        \label{fig:heatmap_1k}
    \end{subfigure}
    \hfill
    \begin{subfigure}{0.23\textwidth}
        \centering
        \includegraphics[width=\linewidth]{Figures/MNISTUncond/t8192_NoGAN/s1t8_ot_cost_matrix_heatmap.png}
        \caption{Teacher 8k Iterations}
        \label{fig:heatmap_8k}
    \end{subfigure}
        
        \caption{
        \textbf{Longer trained teachers are more likely copied. } \textbf{Top two panels:} Thirty image quintuples generated from random initial noise \(z\). Within each panel, from top to bottom, the rows show the teacher 1-step samples \(\Phi_1(z)\), teacher 8-step samples \(\Phi_8(z)\), student one-step samples \(G(z)\), the nearest training datapoints \(x^1(\Phi_8(z))\), and the second-nearest training datapoints \(x^2(\Phi_8(z))\). The upper panel corresponds to a teacher trained for 1024 iterations, and the lower panel corresponds to a teacher trained for 8192 iterations. \textbf{Bottom panels}: Visualization of pairing inefficiency via distance heatmaps. Heatmaps display pairwise $L_2$ distances between teacher $\{\Phi_K(z_i)\}_{i=1}^{100}$ and student $\{G(z_j)\}_{j=1}^{100}$ outputs for students distilled from 1024-step (left) and 8192-step (right) teachers. Red boxes denote the Optimal Transport (OT) pairing, while the diagonal represents the Distillation Transport (DT) pairing. At 1024 teacher iterations, the student remaps the pairings with inefficiency $\Delta_E \approx 1.05$, whereas at 8192 iterations, the student exhibits strong copying behavior with $\Delta_E \approx 0.0367$.
    }
    \label{WellTrainedTeachersLikelyCopied}
\end{figure}

We train teacher diffusion models on the unconditional MNIST dataset for varying numbers of iterations, and initialize student generators from these teacher snapshots for subsequent distillation. We observe that students distilled from later-stage teacher checkpoints preserve a substantially higher proportion of the teacher's original noise--data pairings (see Figure~\ref{WellTrainedTeachersLikelyCopied}, \ref{UncondMNIST 1248K longer train more likely copied}).

This observation is also consistent with our hypothesis. As training progresses, the teacher more accurately resolves the underlying target distribution, reducing excess diffusion-induced variability and progressively capturing finer geometric structure of the data manifold. These increasingly refined structures impose stronger constraints on the student during distillation, thereby reducing the flexibility available for large-scale remapping of noise--data pairings. Consequently, optimization dynamics become more favourable toward the copying behavior.

\subsection{Related Work}
\textbf{Memorization and reproducibility in diffusion models.} 
Prior works studied memorization, generalization, and reproducibility in diffusion models, showing that memorization occurs on excessively small datasets~\cite{emergenceandreproducibility, CDCFM, whyDMdontmemorize}, while reproducibility persists even in strongly generalizing regimes~\cite{understandgeneralizabilityofDMrequires, genofDMarisefromgeometryharmreps, emergenceandreproducibility}. These works focus on similarities between independently trained diffusion models, whereas our work studies copying behavior arising during student distillation.

\textbf{Distillation of diffusion and flow-based models.}
Existing distillation methods for diffusion models include DMD, VSD, Diff-Instruct~\cite{DMD1, DMD2, VSDvideogen, DiffInstruct}, while approaches to distill flow-based models include progressive distillation, consistency models, MeanFlow, and SplitMeanFlow~\cite{ProgressiveDistill, consistencymodels, meanflow, splitmeanflow}. Unlike flow-based distillation, where trajectory-level supervision is directly available, distribution-matching distillation only provides distribution-level supervision, theoretically allowing freedom for remapping noise--data pairings.

\textbf{Manifold structure of data distributions.}
The manifold hypothesis suggests that natural image distributions are concentrated near low-dimensional manifolds ~\cite{scorebasedGendetectmfld}. Prior works studied geometric and topological challenges arising in score-based generative models under this setting, including numerical stability, convergence, and manifold-related artifacts~\cite{scorebasedGendetectmfld, convergenceofDDMundermfld, convergenceofDMundermfldhighdim, deepgenmodelsthroughlensofmfldhyp}. Our observations confirm that such geometric constraints may also influence copying behavior during distillation.

\section{Conclusion}
In this work, we investigate the unexpected copying behavior exhibited by distribution-matching distilled student models, where the student reproduces the teacher's noise--data pairings despite having access only to distribution-level supervision.

Through extensive experiments, we show that copying is not caused by regression or adversarial objectives or by memorization of the training dataset by the teacher model. Empirical evidence instead supports the hypothesis that copying emerges when the student has limited freedom during distillation to deform its distribution while remaining aligned with the data distribution induced by teacher's learned score.

\section{Limitations and Future Work}
While our experiments rule out several intuitive explanations and suggest a geometric driver for copying, our current insights remain primarily correlational and bounded in scope. 

First, although empirical evidence supports that boundary constraints limit the student's deformation freedom, our framework lacks a formal analytical or topological proof. A more rigorous theoretical treatment is required. Future work could leverage high-dimensional geometric descriptors to dynamically track the student's score function during training~\cite{mflddiffusiongeom}, bypassing linear low-dimensional projections.

Second, our evaluation is restricted to compact datasets (MNIST, ImageNet-64). Validation on modern high-resolution text-to-image pipelines remains unproven, and scaling up these findings is a crucial next step. Finally, understanding which structural properties of high-dimensional data encourage copying, such as whether it relates to the teacher learning an approximate optimal transport map, offers a promising avenue for future investigation.

\section*{Acknowledgements}

S.L. is supported by a China Scholarship Council (CSC) - PAG Oxford Scholarship. M.B. is partially supported by the EPSRC Turing AI World-Leading Research Fellowship No. EP/X040062/1 and EPSRC AI Hub No. EP/Y028872/1.

\bibliography{references}
\bibliographystyle{icml2026_fogen}

\newpage

\appendix
\onecolumn

\section{Copying and Remapping Behaviors of DMD Students on Other Datasets}
\label{A. Copying on Other Datasets}
Similar to the toy chessboard experiment, on other low-dimensional synthetic datasets, the distribution-matching distilled student also exhibits \textit{substantial remapping} behavior, even with \(\|G(z)-\Phi_8(z)\| \gg \|\Phi_1(z)-\Phi_8(z)\|\) for some $z$.

\begin{figure}[ht]
  \vskip 0.1in
  \begin{center}
    \centerline{\includegraphics[width=0.75 \columnwidth]{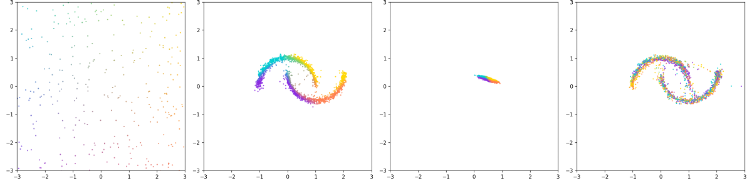}}

    \vskip 0.2in    
    \centerline{\includegraphics[width=0.75 \columnwidth]{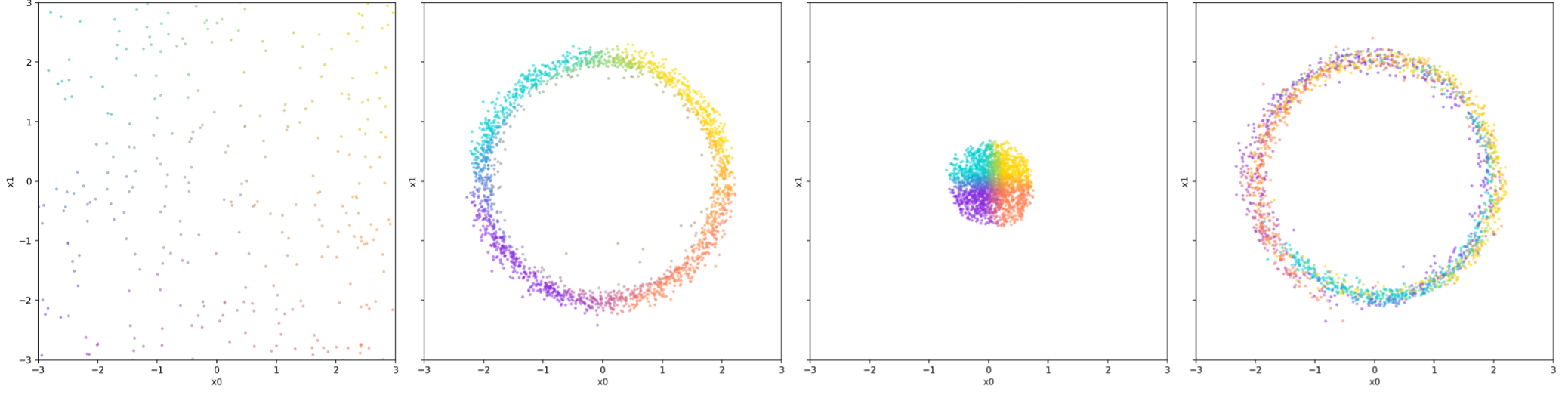}}
    \caption{Noise--data pairing results on low-dim additional artificial datasets. In each row, panels from left to right show the initial Gaussian noise \(z\), the teacher 8-step Euler samples \(\Phi_8(z)\), the teacher 1-step Euler samples \(\Phi_1(z)\), and the student one-step samples \(G(z)\). Points generated from the same initial noise seed \(z\) are assigned the same color across panels. \textbf{Top row:} Distillation results on the 2D double-moons dataset embedded in the first two coordinates of a four-dimensional ambient space. \textbf{Bottom row:} Distillation results on a Gaussian mixture dataset consisting of 32 isotropic four-dimensional Gaussian components with shared standard deviation \(0.2\). The component means are arranged on a circle of radius \(2\) in the first two coordinates of the ambient space.}
    \label{copying on artificial datasets}
  \end{center}
\end{figure}

In contrast, high-dimensional datasets typically exhibit \textit{significantly stronger copying behavior}, which becomes even more pronounced in \textit{conditional} generation settings (Figure~\ref{copying on high-dim datasets: t1t8t32s1 visualization}, \ref{copying on highdim datasets: pca}, and \ref{copying on high dim datasets: displacement compare}). We conjecture that conditioning enables the teacher to resolve finer geometric structures within each class separately, thereby imposing stronger constraints on the student during distillation and reducing its freedom to remap noise--data pairings while preserving the target distribution.

\begin{figure}[H]
\centering

\includegraphics[width=0.9\columnwidth]{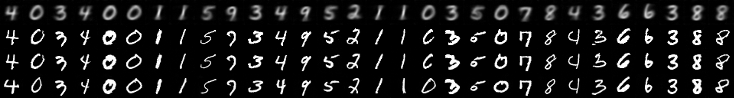}

\vskip 0.08in

\includegraphics[width=0.9\columnwidth]{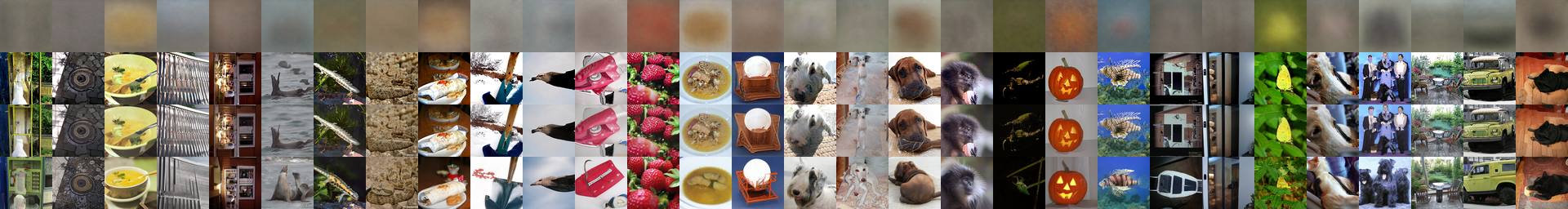}

\vskip 0.08in

\includegraphics[width=0.9\columnwidth]{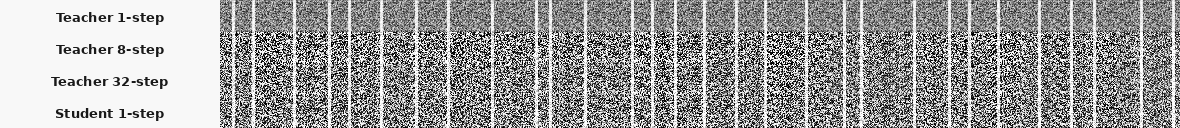}

\caption{
Distillation results on conditional high-dimensional datasets. Panels from top to bottom: conditional MNIST, conditional ImageNet64, and the conditional synthetic MLP-manifold dataset. Each panel contains thirty image quadruples generated from random initial noise seeds \(z\) and randomly assigned classes. From top to bottom within each panel are the teacher 1-step samples \(\Phi_1(z)\), teacher 8-step samples \(\Phi_8(z)\), teacher 32-step samples \(\Phi_{32}(z)\), and student one-step samples \(G(z)\). We use the teacher model provided by~\cite{DMD2} for the conditional ImageNet64 experiment, for deterministic teacher sampling we set Schurn = 0.
}
\label{copying on high-dim datasets: t1t8t32s1 visualization}
\end{figure}

\newpage 
We also remark that the copying behavior is not restricted to natural image datasets. We construct a synthetic high-dimensional dataset by applying a randomly initialized two-layer MLP to 16-dimensional Gaussian noise, followed by superimposing whitened stripes at different spatial locations according to randomly assigned class labels \(i=0,\dots,9\). This produces a distribution in \(\mathbb R^{32\times32\times1}\) supported on a 16-dimensional manifold with 10 conditional classes. We observe similarly strong copying behavior by the student on this dataset (Figure~\ref{copying on high-dim datasets: t1t8t32s1 visualization}, \ref{copying on highdim datasets: pca}, and \ref{copying on high dim datasets: displacement compare}).

\begin{figure}[ht]
\centering

\includegraphics[width=0.7 \columnwidth]{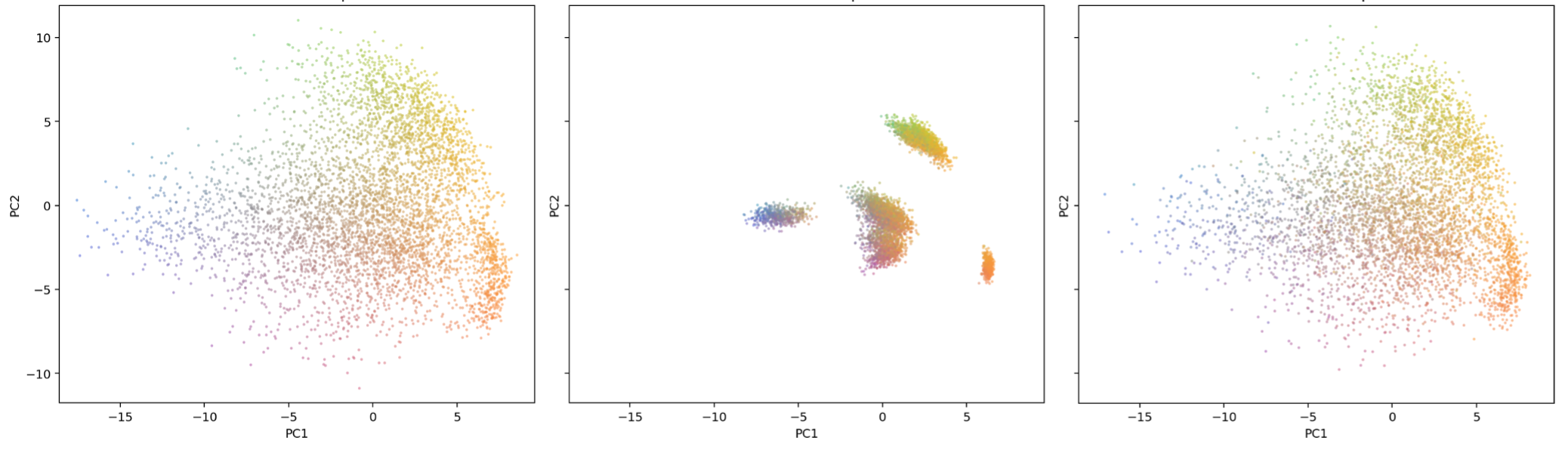}

\vskip 0.08in

\includegraphics[width=0.7\columnwidth]{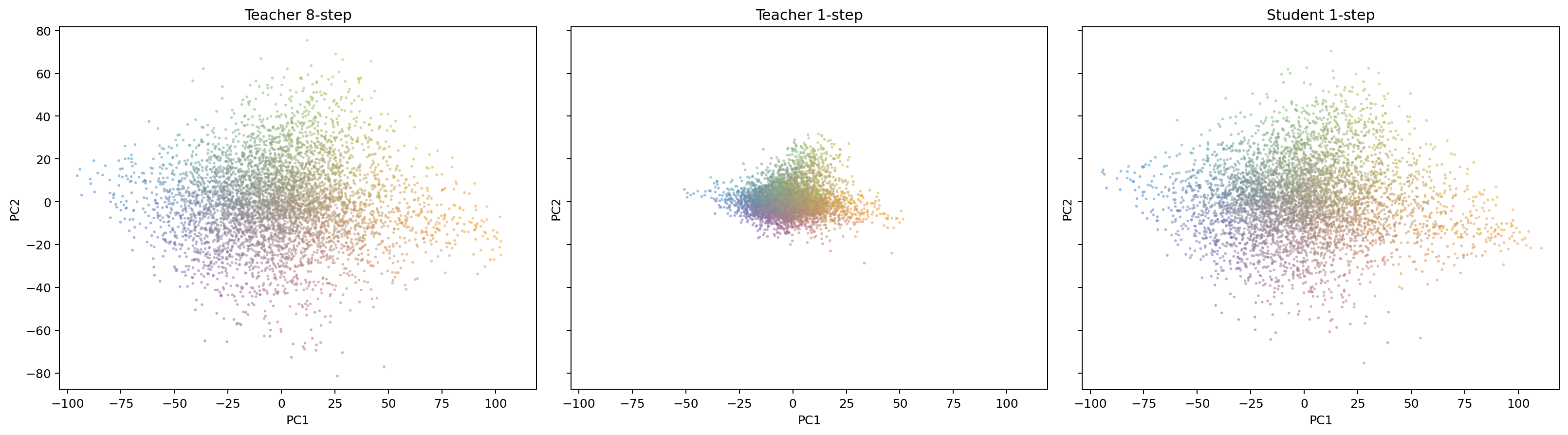}

\vskip 0.08in

\includegraphics[width=0.7\columnwidth]{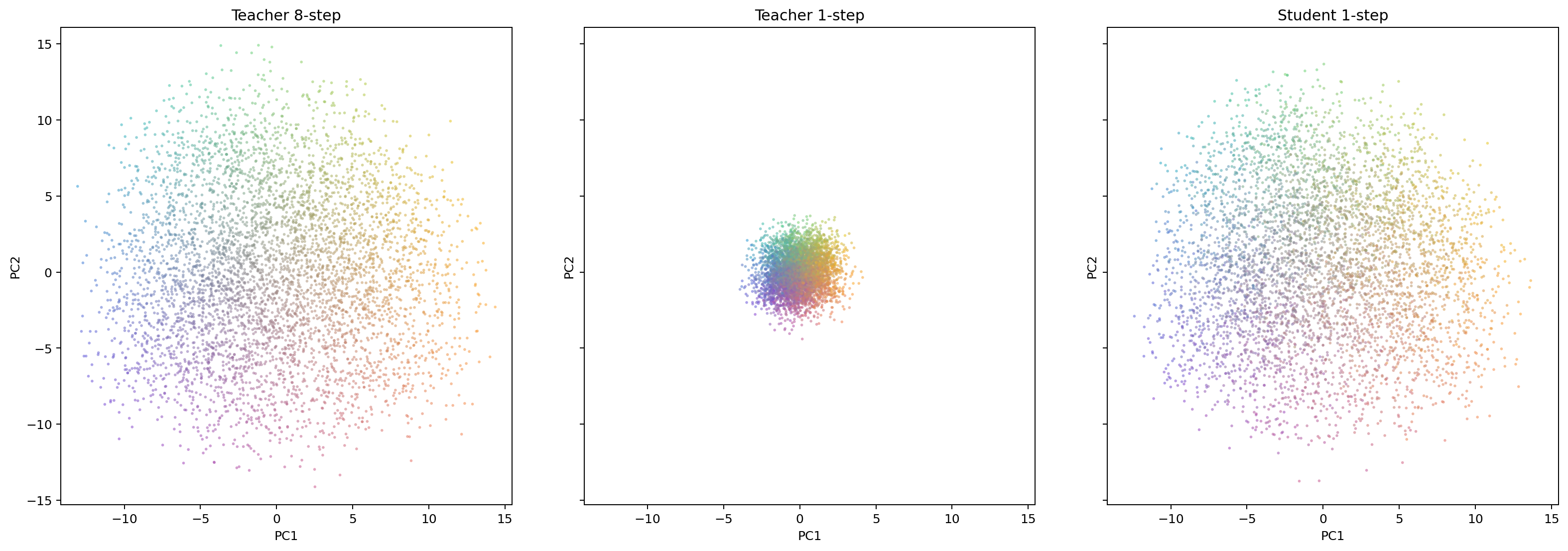}

\caption{
Visualization of noise--data pairing relationships on conditional high-dimensional datasets. Rows from top to bottom: conditional MNIST, conditional ImageNet64, and the conditional synthetic MLP-manifold dataset. Each row visualizes 2000 triplets \((\Phi_8(z), \Phi_1(z), G(z))\) projected onto the two leading principal components computed from the teacher 8-step samples \(\Phi_8(z)\).
}
\label{copying on highdim datasets: pca}
\end{figure}

\begin{figure}[H]
\centering

\includegraphics[width=0.25\textwidth]{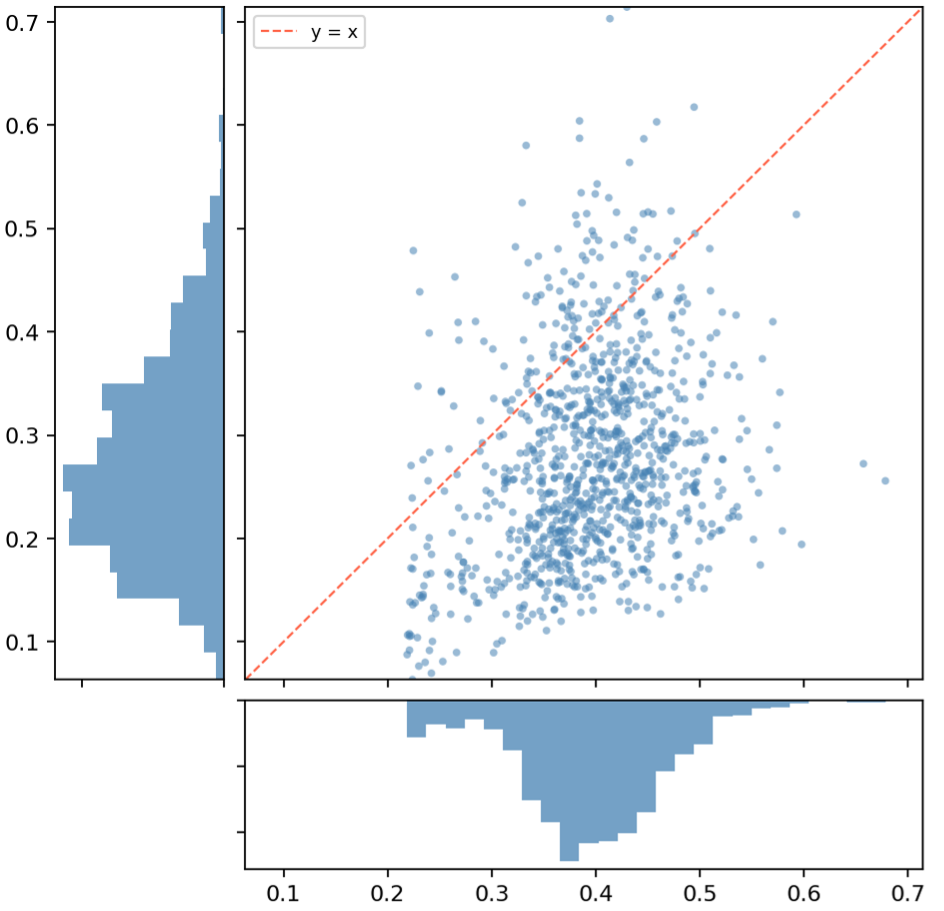}
\hspace{0.03in}
\includegraphics[width=0.26\textwidth]{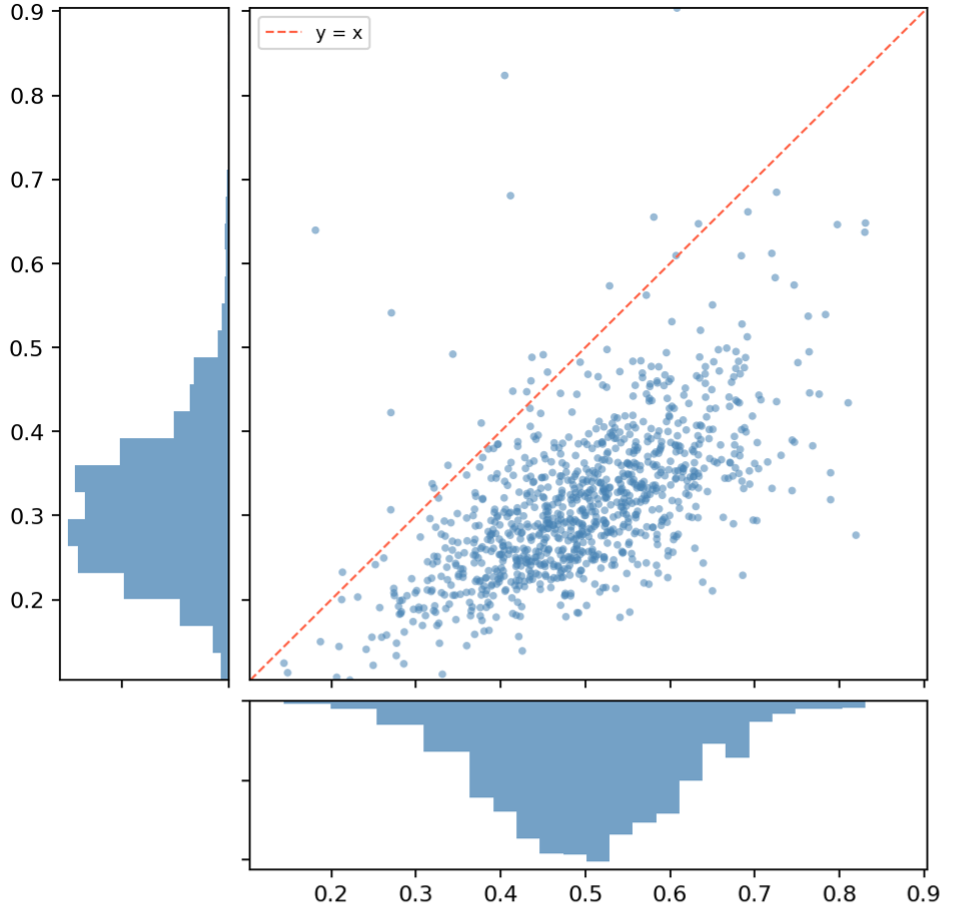}
\hspace{0.03in}
\includegraphics[width=0.25\textwidth]{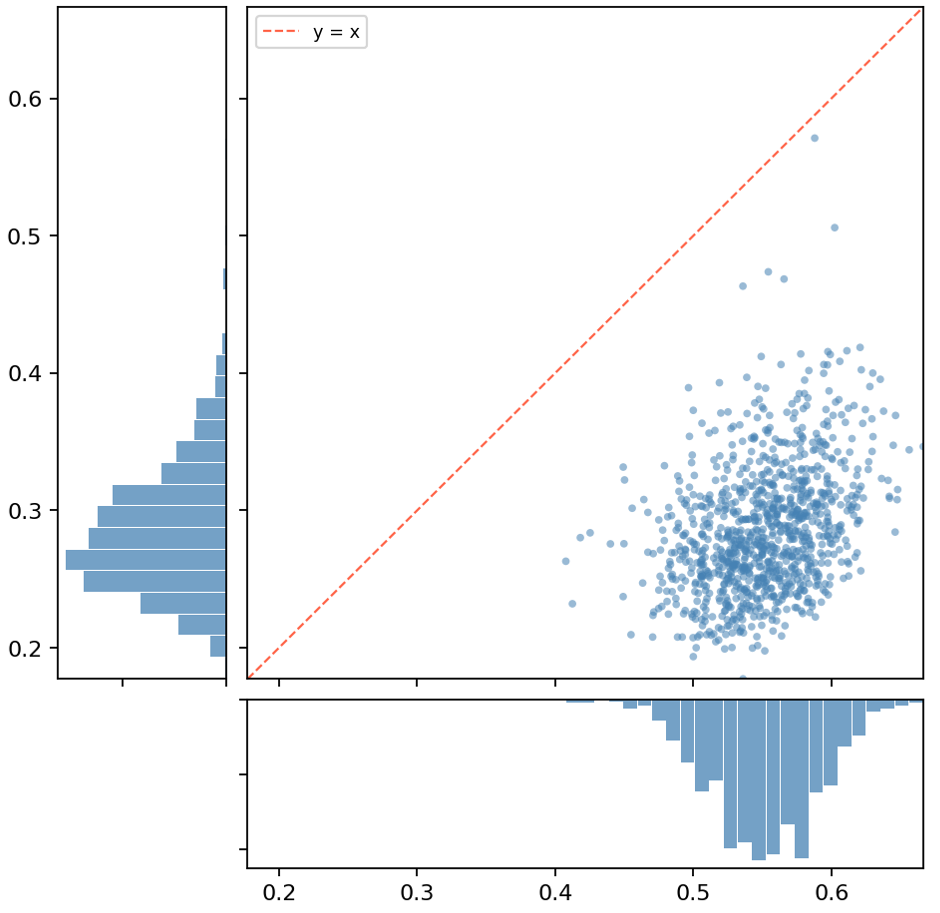}

\caption{
Comparison of student copying behavior on conditional high-dimensional datasets. \textbf{Left:} conditional MNIST. \textbf{Middle:} conditional ImageNet64. \textbf{Right:} conditional synthetic MLP-manifold dataset. In all plots, the horizontal axis shows \(\|\Phi_1(z)-\Phi_8(z)\|\), while the vertical axis shows \(\|G(z)-\Phi_8(z)\|\). Axes are proportionally rescaled where appropriate to reflect pixel-level or channel-wise discrepancy magnitudes. Across all three datasets, students distilled from sufficiently trained teachers exhibit strong copying behavior.
}
\label{copying on high dim datasets: displacement compare}
\end{figure}

\clearpage
\newpage
\section{Copying on Teachers Trained for Various Lengths on Unconditional MNIST}
\label{B. Copying of various length teacher on UncondMNIST}
We train teacher diffusion models on the unconditional MNIST dataset for 1024, 2048, 4096, and 8192 iterations, and distill corresponding student models initialized from these checkpoints. Each student is subsequently distilled for 50K iterations. We observe that students initialized from later-stage teacher checkpoints generally preserve a higher proportion of the teacher's original noise--data pairings.

\begin{figure}[H]
    \centering
        \includegraphics[width=0.8\textwidth]{Figures/MNISTUncond/t1024_NoGAN/memorizations/1.png}
        
        \vskip 0.08in

        \includegraphics[width=0.8\textwidth]{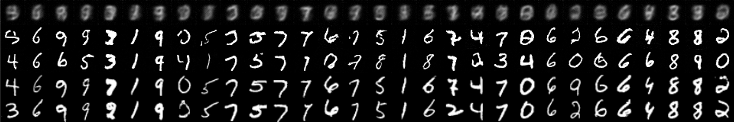}
        
        \vskip 0.08in

        \includegraphics[width=0.8\textwidth]{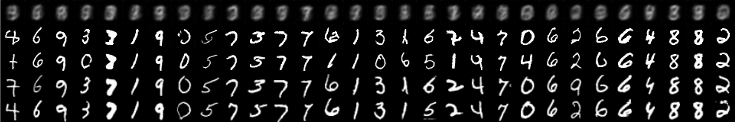}
        
        \vskip 0.08in
        
        \includegraphics[width=0.8\textwidth]{Figures/MNISTUncond/t8192_NoGAN/memorization/3.png}
        
        \vskip 0.2in
        
        \includegraphics[width=0.22\textwidth]{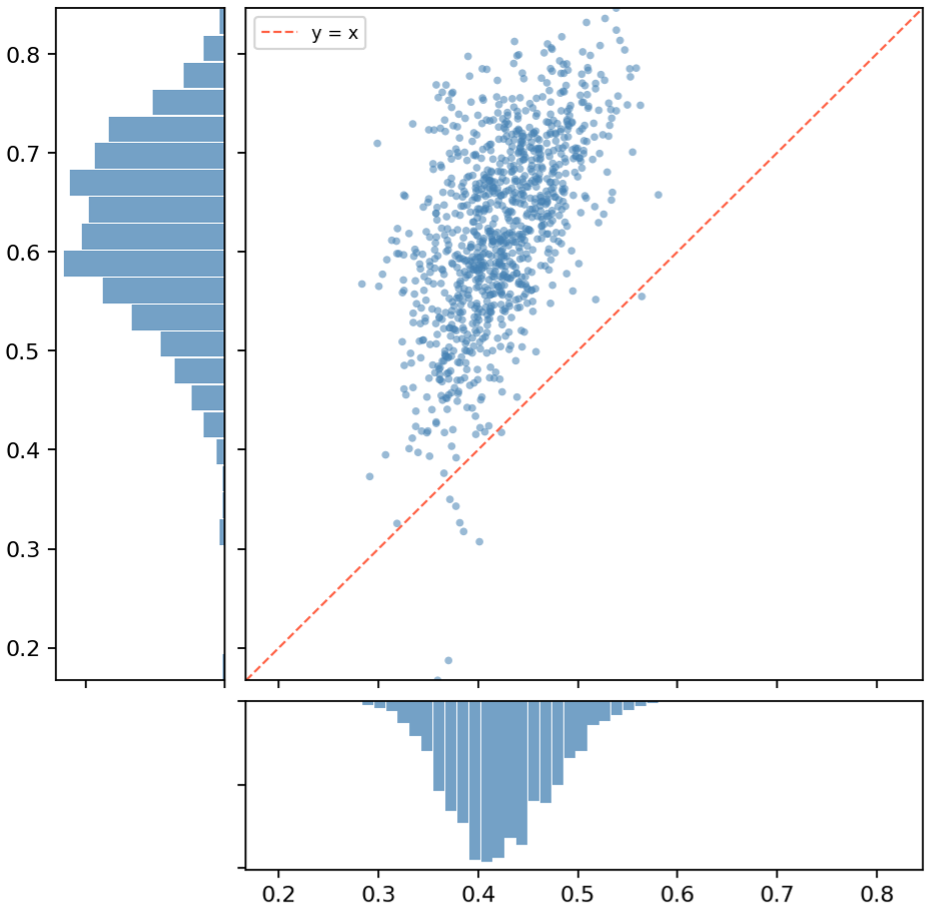}
        \hspace{0.01\textwidth}
        \includegraphics[width=0.22\textwidth]{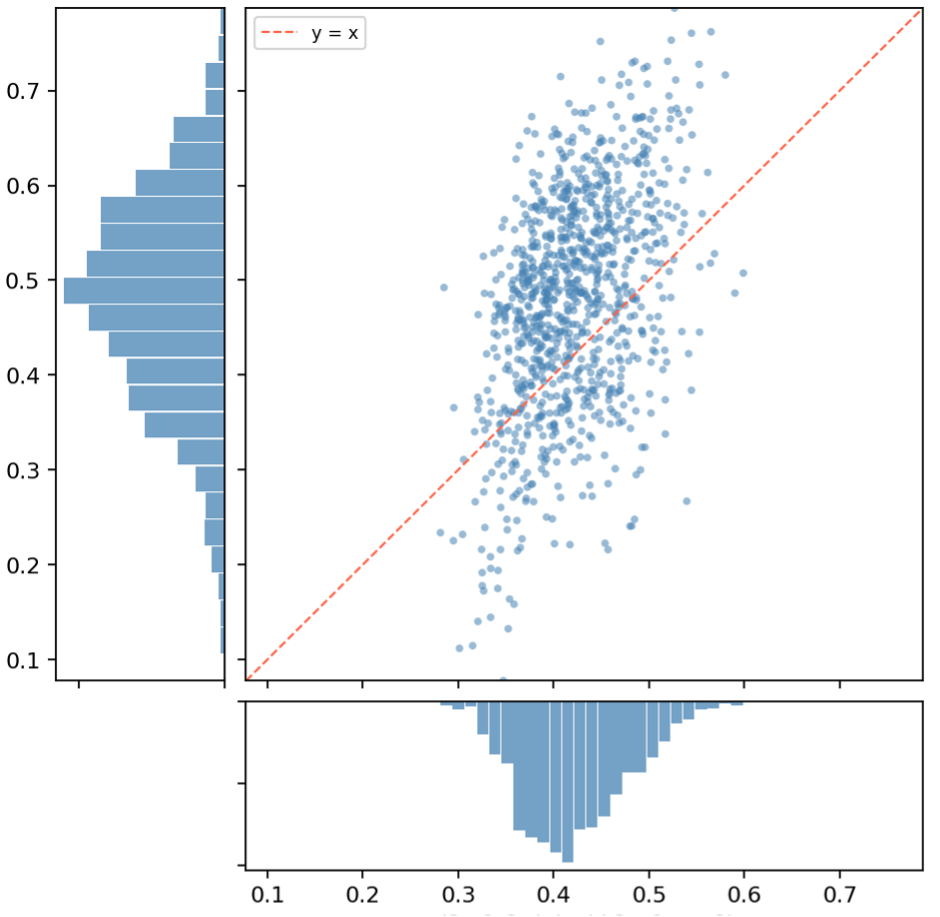}
        \hspace{0.01\textwidth}
        \includegraphics[width=0.22\textwidth]{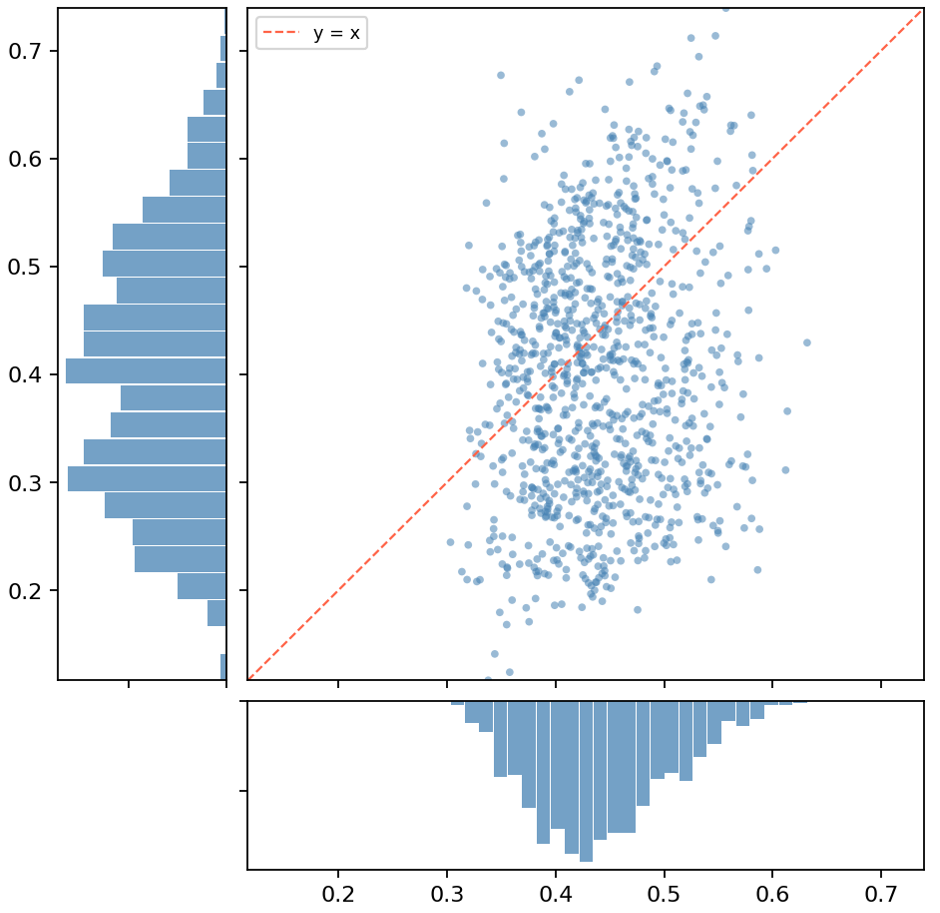}
        \hspace{0.01\textwidth}
        \includegraphics[width=0.22\textwidth]{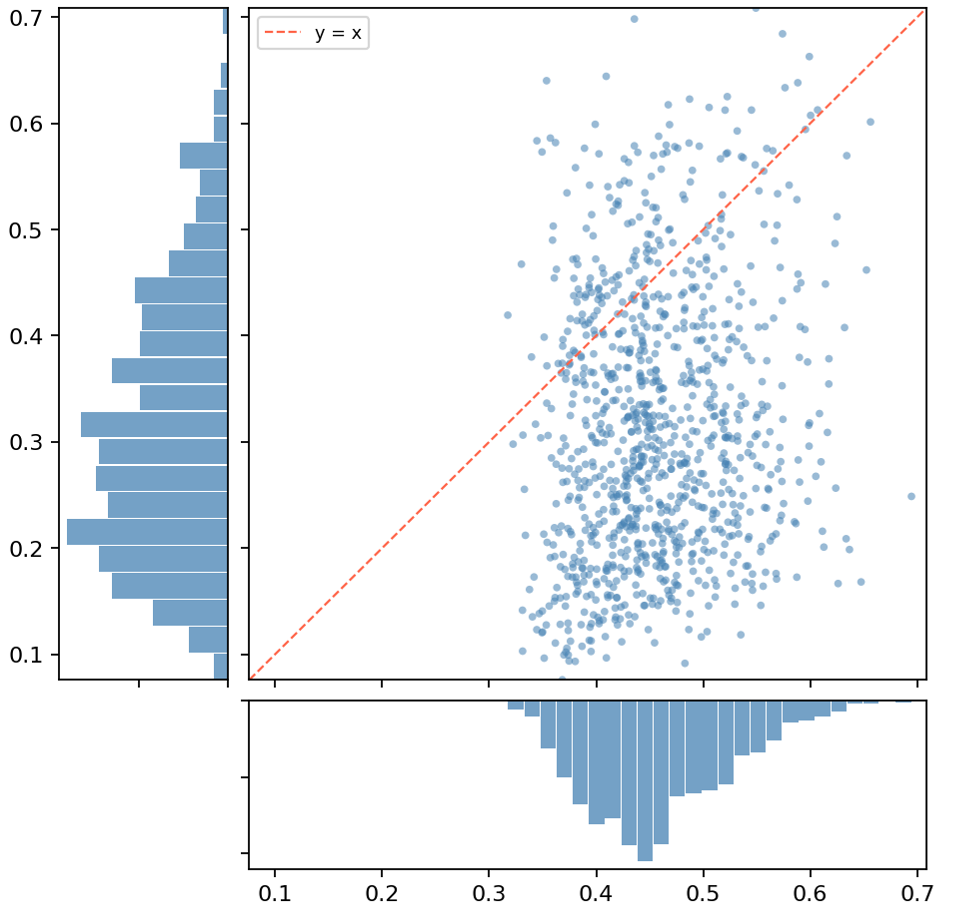}

        \caption{
        \textbf{Copying is more pronounced on longer-trained teachers. } Distillation results of teachers trained for various lengths on the unconditional MNIST dataset. \textbf{Top: }The four panels correspond to teachers trained for 1024, 2048, 4096, and 8192 iterations. For each panel, thirty image quintuples are generated from random initial noise seeds \(z\). From top to bottom the five rows show the teacher 1-step samples \(\Phi_1(z)\), teacher 8-step samples \(\Phi_8(z)\),  student one-step samples \(G(z)\), closest training point \(x^1(\Phi_8(z))\) and second closest training point \(x^2(\Phi_8(z))\). \textbf{Bottom:} Comparison of copying behaviors for students distilled from teacher checkpoints trained for 1024, 2048, 4096, and 8192 iterations. In each plot the horizontal axis shows \( \|\Phi_1(z)-\Phi_8(z)\|\) and the vertical axis shows \(\|G(z)-\Phi_8(z)\|\), both axes rescaled by $1/32$. The pairing inefficiencies of students distilled at these snapshots are $\Delta_E=1.05, 0.389, 0.237$ and $0.0367$ respectively, showing stronger copying on teachers trained for larger number of iterations.
        }
    \label{UncondMNIST 1248K longer train more likely copied}
\end{figure}

\clearpage
\section{Derivations}
\label{C.Derivations}
\begin{lemma}
\label{lem:usefullemma}
Let $G_\theta$ and $G_{\theta'}$ be two single-step student generators, potentially with different architectures (or with the same architecture but different parameters). Whenever $G_\theta(z) \stackrel{d}{=} G_{\theta'}(z)$, it follows that~$L_{DM}(\theta) = L_{DM}(\theta')$, even though in general $\nabla_\theta L_{DM}(\theta) \neq \nabla_\theta L_{DM}(\theta')$. Consequently, the stochastic optimization dynamics can drive students to converge toward different pairings with vastly different pairing inefficiencies $\Delta_E$, despite achieving similar distributional fidelity. 
\end{lemma}

\begin{proof}
    Given \(G_\theta(z)\stackrel d = G_{\theta'}(z)\) equal in distribution, since $\varepsilon\sim \mathcal N(0, \mathbf I)$ is independent of $z\sim \mathcal N(0,\sigma^2(T)\mathbf I)$, we have that \[G_\theta(z)+\sigma(t)\varepsilon \stackrel d = G_{\theta'}(z)+\sigma(t)\varepsilon.\] But by definition of the student probability path, this is just $p_{\theta,t}  = p_{\theta',t}$ equal in distribution for all $t\in T$. Noting that KL-divergence is only dependent on the distribution supplied, it follows that 
    
    \begin{align*}
        L_{\text{DM}}(\theta) & := \int w(t) \text{KL}(p_{\theta, t} \| p_t) \, \mathrm dt.  \\ &=\int w(t) \text{KL}(p_{\theta', t} \| p_t) \, \mathrm dt. \\ &=L_{\text{DM}}(\theta').
    \end{align*}
    Finally, \(\nabla_\theta L_{\mathrm{DM}}(\theta)\neq\nabla_\theta L_{\mathrm{DM}}(\theta')\) in general because $\partial G_\theta(z)/\partial \theta$ generally takes different values at different $\theta$ and $\theta'$. 

    To illustrate how two student models can achieve similar (identical) distributional fits while learning pairings with vastly different efficiencies, consider the following geometric construction. Let the target distribution $p_{\text{data}}$ be a singular measure in $\mathbb{R}^2$ supported on a slightly deformed circle $(1+\delta) x^2+y^2=1$, where $\delta \ll 1$. This distribution is defined by pushing forward the isotropic Gaussian noise $(x, y) \sim N(0, \sigma^2 I)$ via the mapping $T$:$$T(x, y) = \frac{(x, y)}{\sqrt{(1+\delta) x^2+y^2}}.$$
    
    Assume the teacher model has learned the distribution perfectly, such that $\Phi_K(x, y) := T(x, y)$. Now, consider a restricted class of student models $G_\theta$ which first applies a rotation $\theta$ to the noise space prior $\mathcal N(0,\sigma^2(T)\mathbf I)$ then projects it onto the unit circle: $$G_\theta(x, y) = \frac{(x \cos \theta - y \sin \theta, x \sin \theta + y \cos \theta)}{\sqrt{x^2+y^2}}.$$
    
    Because the Gaussian noise source is rotationally invariant, $G_\theta$ generates the same output distribution $p_{\text{data}}$ for any $\theta$. Consequently, both $\theta=0$ and $\theta=\pi$ are global optimizers of the distribution matching loss:$$L_\mathrm{DM}(\theta=0) = L_\mathrm{DM}(\theta=\pi)\approx 0.$$
    
    However, these two solutions exhibit fundamentally different pairing efficiencies:\begin{itemize}\item At $\theta=0$, the student's orientation aligns perfectly with the teacher's, resulting in~$\Delta_E(\Phi_K, G_0) \approx 0$, manifesting clear copying behavior.\item At $\theta=\pi$, while the generated distribution is identical, the student reverses the orientation of the mapping relative to the teacher, resulting in a significantly larger $\Delta_E(\Phi_K, G_\pi)\gg 0$.\end{itemize}

    We note this example also works to show that in general \(\nabla_\theta L_{\mathrm{DM}}(\theta)\neq\nabla_\theta L_{\mathrm{DM}}(\theta')\) for $\theta '=\theta +\pi$ for $\theta \notin\{0, \pi\}.$

\end{proof}

\clearpage

\begin{lemma}[Properties of Pairing Inefficiency]
The pairing inefficiency $\Delta_E$ is non-negative ($\Delta_E \geqslant 0$) and invariant under uniform scaling of all measures, i.e., $\Delta_E(c \Phi_K, c G) = \Delta_E(\Phi_K, G)$ for all $c > 0$, providing a robust measure for comparing copying behaviors across scales. Furthermore, the empirical estimator $\Delta_E^{(N)}$ is consistent, such that $\Delta_E^{(N)} \to \Delta_E$ almost surely as $N \to \infty$.
\end{lemma}

\begin{proof}
Let $\pi_{DT}$ denote the joint distribution of pairs $(\Phi_K(z), G(z))$ induced by the shared noise source $z \sim p_z$. By construction, $\pi_{DT}$ is a valid coupling in the set of all couplings $\Gamma(p_\Phi, p_G)$. Since the optimal transport cost $E_{OT}$ is defined as the infimum over this set, it follows that:
\[
E_{DT}(\Phi_K, G) = \int \|x - y\|_2^2 \, d\pi_{DT}(x, y) \geqslant \inf_{\pi \in \Gamma} \int \|x - y\|_2^2 \, d\pi(x, y) = E_{OT}(\Phi_K, G).
\]
Therefore, $$\Delta_E = E_{DT} / E_{OT} -1 \geqslant 0.$$

To show scale invariance, consider a constant $c > 0$. For any coupling $\pi \in \Gamma(p_\Phi, p_G)$, the scaled coupling $\pi_c = (c, c)_\# \pi$ belongs to $\Gamma(p_{c\Phi}, p_{cG})$. The scaled optimal transport cost is:
$$
\begin{aligned}
E_{OT}(c \Phi_K, c G) &= \min_{\pi_c \in \Gamma(p_{c\Phi}, p_{cG})} \int \|x - y\|_2^2 \, d\pi_c(x, y) \\
&= \min_{\pi \in \Gamma(p_{\Phi}, p_{G})} \int \|cx - cy\|_2^2 \, d\pi(x, y) \\
&= c^2 \min_{\pi \in \Gamma(p_{\Phi}, p_{G})} \int \|x - y\|_2^2 \, d\pi(x, y) \\
&= c^2 E_{OT}(\Phi_K, G)
\end{aligned}
$$
Similarly, $E_{DT}(c \Phi_K, c G) = \int \|cx - cy\|_2^2 \, d\pi_{DT}(x, y) = c^2 E_{DT}(\Phi_K, G)$. 
Because $\Delta_E$ is defined as a relative ratio, the factor $c^2$ cancels, yielding scale invariance $$\Delta_E(c \Phi_K, c G) = \Delta_E(\Phi_K, G).$$

Finally, by Varadarajan's Theorem, the empirical measures $p_{\Phi}^{(N)}$ and $p_{G}^{(N)}$ converge weakly to $p_{\Phi}$ and $p_{G}$ almost surely. Given that the images of $\Phi_K$ and $G$ are bounded within a subset of compact support (due to the bounded non-linearities of the generators), the $L^2$ Wasserstein distance is continuous with respect to weak convergence. Consequently, the empirical estimator $\Delta_E^{(N)}$ converges almost surely to the population value $\Delta_E$.
\end{proof}

\clearpage
\section{Implementation Details}
\textbf{Network Architectures and Preconditioning.} For our primary high-dimensional benchmarks, we utilize the conditional ImageNet-64 teacher model provided by \citet{DMD2}, which employs the ADM architecture \cite{dhariwal2021diffusion} with a base width of $C_{base}=192$ and $label\_dim=1000$. The student model is an architectural copy of this teacher, initialized directly from its pre-trained weights to observe the distillation phenomena. For the unconditional MNIST experiments ($32\times32$, 1-channel), we adapt the same ADM architecture but reduce the capacity to $C_{base}=64$ and set $label\_dim=0$ to ensure strictly unconditional distillation. Both image-based models share a consistent UNet configuration, featuring a channel multiplier of $[1, 2, 3, 4]$, three residual blocks per resolution, and self-attention layers at resolutions of $32, 16,$ and $8$. EDM preconditioning~\cite{EDM} is applied to scale inputs, outputs, and skip connections by $\sigma$-dependent factors, maintaining unit variance throughout the distillation process. Finally, for the synthetic experiments on the 2D checkerboard manifold embedded in 4D space, we utilize a MLP architecture with $5$ hidden layers of $384$ hidden units, with 32-dimensional Fourier time embedding. 

\textbf{Distillation Dynamics and Optimization. }Student models are initialized from their respective pre-trained teacher weights. A central feature of our distillation is the $1:5$ update ratio between the generator ($G_\theta$) and the fake score model ($s_\psi$). The fake score model is optimized at every iteration to provide a high-quality surrogate of the evolving student's actual score $\nabla_x \log ((G_\theta(z)\ast \mathcal N(0,t))(x))$ while the generator is updated via the distribution matching loss every five iterations. We utilize the Karras noise schedule~\cite{EDM} with parameters $\sigma_{\min}=0.002, \sigma_{\max}=80,$ and $\rho=7$. To ensure numerical stability and avoid training on uninformative noise levels at the extreme ends of the SDE trajectory, timesteps are sampled from a restricted interval of $[2\%, 98\%]$. All models are trained using the AdamW optimizer with a learning rate of $2\times10^{-6}$, weight decay of $0.01$, and a $500$-step linear warmup. Distillation for ImageNet-64 is conducted on a single NVIDIA H100 GPU for $50,000$ iterations using batch size of 32. MNIST and synthetic toy experiments are performed on the same GPU with batch sizes of $32$ and $512$ respectively.

\section{Additional Figures}
\label{D. Additional Figures}
\begin{figure}[ht]
  \vskip 0.1in
  \begin{center}
    \centerline{\includegraphics[width=0.6 \columnwidth]{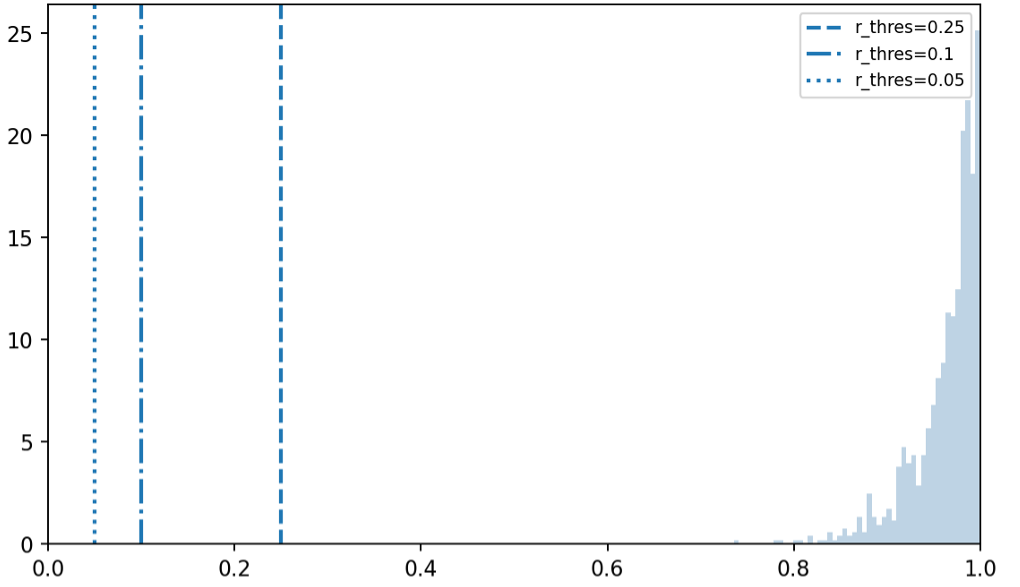}}

    \caption{Distribution of memorization distance ratios $r(\Phi_8(z)):=\|x^2(\Phi_8(z))-\Phi_8(z)\| /\|x^1(\Phi_8(z))-\Phi_8(z)\|$ for the teacher model trained for 8192 iterations on unconditional MNIST.}
    \label{teacher do not memorize}
  \end{center}
\end{figure}

This shows that the teacher model is not memorizing any of the training datapoints.

\end{document}